\documentclass[preprint]{uai2026}
                        
\usepackage{xcolor}
\usepackage[american]{babel}
\usepackage[]{color-edits}
\addauthor{rd}{purple}
\addauthor{sl}{Green}
\usepackage{subcaption}
\usepackage{natbib} 
\usepackage{mathtools} 
\usepackage{amsfonts}   
\usepackage{amsthm}
\newtheorem{theorem}{Theorem}
\newtheorem{remark}{Remark}
\newtheorem{corollary}{Corollary}
\newtheorem{proposition}{Proposition}
\usepackage{booktabs} 
\usepackage{tikz} 
\usepackage{microtype}
\usepackage{algorithm}
\usepackage{algorithmic}
\usepackage{enumitem}
\usepackage[dvipsnames]{xcolor}
\usepackage{tcolorbox}
\definecolor{darkgreen}{rgb}{0.05, 0.5, 0.25}
\newenvironment{greenbox}[1][]{
    \begin{tcolorbox}[
        left=5pt, right=5pt,  
        colback=darkgreen!1,     
        colframe=darkgreen!150,    
        colbacktitle=darkgreen!10,
        coltitle=black,      
        fonttitle=\bfseries, 
        title=%
         {\centering \if\relax\detokenize{#1}\relax
          \else #1\fi }     
    ]
}{
    \end{tcolorbox}
}

\title{KMM-CP: Practical Conformal Prediction under Covariate Shift via Selective Kernel Mean Matching}

\author[1]{\href{mailto:<sl160@illinois.edu>}{Siddhartha Laghuvarapu}{}}
\author[1]{Rohan Deb}
\author[1]{Jimeng Sun}
\affil[1]{%
    Siebel School of Computing and Data Science, University of Illinois Urbana-Champaign, IL, USA.\\
}
  
\begin{document}
\maketitle

\begin{abstract}
    Uncertainty quantification is essential for deploying machine learning models in high-stakes domains such as scientific discovery and healthcare. Conformal Prediction (CP) provides finite-sample coverage guarantees under exchangeability, an assumption often violated in practice due to distribution shift. Under covariate shift, restoring validity requires importance weighting, yet accurate density-ratio estimation becomes unstable when training and test distributions exhibit limited support overlap.
    We propose KMM-CP, a conformal prediction framework based on Kernel Mean Matching (KMM) for covariate-shift correction. We show that KMM directly controls the bias–variance components governing conformal coverage error by minimizing RKHS moment discrepancy under explicit weight constraints, and establish asymptotic coverage guarantees under mild conditions. We then introduce a selective extension that identifies regions of reliable support overlap and restricts conformal correction to this subset, further improving stability in low-overlap regimes. Experiments on molecular property prediction benchmarks with realistic distribution shifts show that KMM-CP reduces coverage gap by over 50\% compared to existing approaches. The code is available at \url{https://github.com/siddharthal/KMM-CP}{}

\end{abstract}

\section{Introduction}
Reliable uncertainty quantification is essential for deploying machine learning models in high-stakes domains such as scientific discovery and healthcare. Conformal Prediction (CP) \cite{vovk2005} provides a principled framework for constructing prediction sets with finite-sample marginal coverage guarantees under minimal distributional assumptions. In standard split conformal prediction, a trained model is evaluated on a held-out \emph{calibration set}, whose prediction residuals are used to determine a quantile threshold for forming prediction sets. These guarantees rely on exchangeability between calibration and test data—an assumption that is frequently violated in real-world settings due to distribution shift \cite{tibshirani2019,barber2023conformal}. Extensions of conformal prediction to non-exchangeable or distribution-shift settings have been studied in recent work.

A common and practically relevant form of distribution shift is \emph{covariate shift}, where the marginal distribution of inputs changes while the conditional distribution $Y \mid X$ remains stable \cite{sugiyama2007covariate}. This setting is especially relevant in scientific discovery, where the data-generating process is consistent under comparable experimental conditions \cite{laghuvarapu2023codrug,fannjiang2022conformal}. Under covariate shift, weighted conformal prediction restores validity by reweighting calibration samples using the density ratio between test and calibration covariate distributions \cite{tibshirani2019}. When the ratio is known or accurately estimated, coverage guarantees are recovered.

In practice, density ratio estimation can be unstable, particularly in moderate or high-dimensional settings and under limited support overlap. Ratios may be estimated via probabilistic classification \cite{bickel2009discriminative}, direct density-ratio fitting methods such as KLIEP or LSIF \cite{sugiyama2008kliep,kanamori2009least}, or kernel density estimation \cite{silverman1986density}. Classifier-based methods depend on well-calibrated probability estimates and can yield extreme weights when the test distribution assigns mass to regions poorly represented in calibration data, while kernel density approaches suffer from the curse of dimensionality. In both cases, heavy-tailed or highly concentrated importance weights reduce the effective sample size (ESS), leading to unstable quantile estimation and degraded coverage.

In this work, we propose \emph{KMM-CP}, a framework for conformal prediction under covariate shift based on Kernel Mean Matching (KMM) \cite{huang2007correcting}. Rather than explicitly estimating density ratios, KMM aligns the weighted calibration distribution with the test covariate distribution in a reproducing kernel Hilbert space (RKHS) by minimizing the Maximum Mean Discrepancy (MMD) \cite{gretton2012kernel}. The resulting weights are obtained via a constrained optimization that enforces boundedness, thereby controlling variance and improving stability in high-dimensional settings. Although KMM-based reweighting has appeared in recent work \cite{almeida2025high}, its connection to conformal coverage has not been systematically analyzed. In particular, prior work does not explicitly relate moment-matching quality and effective sample size to coverage behavior, nor compare KMM with alternative density-ratio estimators from a stability perspective.

We further introduce a selective extension of KMM for low-overlap or disjoint-support settings. When test regions lack calibration support, enforcing global moment matching can induce extreme weights and collapse the ESS. Our formulation jointly optimizes calibration weights and target selection variables, restricting correction to regions of shared support. This improves the bias–variance tradeoff and limits uncertainty quantification to regimes where shift correction is reliable. 

Empirically, we evaluate KMM-CP on molecular property prediction tasks with substantial covariate shift. In high-dimensional settings, our method consistently reduces coverage gap and improves efficiency relative to classifier- and kernel density–based baselines, demonstrating robustness in low-overlap regimes. Our main contributions are as follows.
\begin{greenbox}[Main Contributions]
\begin{enumerate}[itemsep=0pt, topsep=0pt]
    \item We propose KMM-CP, a practical framework for conformal prediction under covariate shift that uses Kernel Mean Matching (KMM) to align calibration and test distributions via kernel mean embeddings, improving stability without explicit density estimation.
    \item We introduce a selective extension for low-overlap regimes that restricts correction to regions of shared support, improving effective sample size and calibration reliability.
    \item We provide analysis connecting moment-matching quality and effective sample size to conformal coverage under distributional shift, explaining the improved stability of KMM-based reweighting.
    \item We evaluate KMM-CP on molecular property prediction benchmarks with substantial covariate shift, achieving over 50\% reduction in coverage gap compared to existing baselines.
\end{enumerate}
\end{greenbox}
\section{Preliminaries}
\subsection{Conformal Prediction Framework}
Conformal Prediction (CP) is a framework for constructing prediction sets with finite-sample coverage guarantees.
We consider a classification setup where each data point is $Z = (X, Y)$ with $X \in \mathbb{R}^d$ and $Y \in [K] = \{1, \ldots, K\}$.
Beyond a point estimate from a base classifier $f$, we want a confidence level encoded as a prediction set $\hat{C}(X) \subseteq [K]$.
The main goal is \emph{valid coverage}: for a target level $\alpha \in (0, 1)$ (e.g.\ 0.1), the set should contain the true label with at least $1-\alpha$ probability.
Formally, for a new test point $(X_{N+1}, Y_{N+1})$, $\hat{C}$ is $(1-\alpha)$-valid if
\begin{equation}
\label{eq:valid-coverage}
\mathbb{P}\bigl\{ Y_{N+1} \in \hat{C}(X_{N+1}) \bigr\} \geq 1 - \alpha.
\end{equation}

\paragraph{Split conformal prediction.}
In practice, the coverage guarantee in \eqref{eq:valid-coverage} is typically attained using split (inductive) conformal prediction, which partitions the data into a training set and a calibration set $D_{\text{cal}}$.
A base predictor is fit on the training set, and conformity scores on the calibration set are used to choose a threshold so that the resulting set-valued predictor attains the target coverage.
Under exchangeability of the calibration and test points, the procedure is distribution-free and achieves the validity condition~\eqref{eq:valid-coverage}; see \citet[Theorem 2.1]{vovk2005}.

\subsection{Conformal Prediction under Covariate Shift}

Standard conformal prediction assumes calibration and test data are i.i.d.\ (or exchangeable), an assumption that is rarely realistic in practice. In many scientific settings, however, while the marginal distribution of covariates $X$ may shift between calibration and test time, the conditional distribution $Y \mid X$ remains stable. For example, molecular properties such as target activity or physicochemical characteristics are governed by underlying physical laws and typically remain consistent under comparable experimental conditions. This motivates the \emph{covariate shift} setting, where the marginal law of $X$ changes but the conditional law of $Y \mid X$ is invariant.

Formally, let $D_{\mathrm{cal}}=\{(X_i,Y_i)\}_{i=1}^N$ denote the calibration sample and $(X_{N+1},Y_{N+1})$ the test point. Under covariate shift, $P^{\mathrm{cal}}_{Y\mid X}=P^{\mathrm{test}}_{Y\mid X}$, so the joint laws decompose as
\begin{align}
\label{eq:cal-dist}
P^{\mathrm{cal}}(dx,dy) &= P_X^{\mathrm{cal}}(dx)\,P_{Y\mid X}^{\mathrm{cal}}(dy\mid x), \\
\label{eq:test-dist}
P^{\mathrm{test}}(dx,dy) &= P_X^{\mathrm{test}}(dx)\,P_{Y\mid X}^{\mathrm{cal}}(dy\mid x).
\end{align}
The calibration data satisfy $(X_i,Y_i)\stackrel{\mathrm{i.i.d.}}{\sim}P^{\mathrm{cal}}$, while $(X_{N+1},Y_{N+1})\sim P^{\mathrm{test}}$.



\paragraph{Importance weighting under covariate shift.}
Under covariate shift, one typically assumes $P_X^{\text{test}} \ll P_X^{\text{cal}}$, so the likelihood ratio
\[
w(x) := \frac{dP_X^{\text{test}}}{dP_X^{\text{cal}}}(x)
\]
is well-defined. Tibshirani et al.~\citep{tibshirani2019} view this as a special case of \emph{weighted exchangeability}:
independent (but not identically distributed) samples are weighted exchangeable with weights given by appropriate
Radon--Nikodym derivatives \citep[Lemma~2]{tibshirani2019}. In particular, if $Z_1,\dots,Z_N \sim P^{\text{cal}}$ and
$Z_{N+1}\sim P^{\text{test}}$, then one may take $w_i \equiv 1$ for $i\le N$ and $w_{N+1}(x,y)=w(x)$.
Weighted conformal prediction then forms a cutoff as a \emph{weighted quantile} of the calibration scores, where the
weight $w(x)$ of the candidate test point enters through the weighted conformal rank/quantile construction
\citep[Lemma~3 and Theorem~2]{tibshirani2019}. Using this weighted cutoff (instead of the usual unweighted cutoff)
yields prediction sets with the desired marginal coverage under $P^{\text{test}}$.

\begin{theorem}[Coverage under covariate shift {\citep[Theorem~2]{tibshirani2019}}]
\label{thm:weighted-cp}
Assume the covariate shift model \eqref{eq:cal-dist}--\eqref{eq:test-dist} and $P_X^{\text{test}} \ll P_X^{\text{cal}}$.
Let $w(x)=\frac{dP_X^{\text{test}}}{dP_X^{\text{cal}}}(x)$, and construct the weighted conformal prediction set $\hat C$
using calibration nonconformity scores and the weighted-quantile rule of \citet[Theorem~2]{tibshirani2019} with
weights $w_i\equiv 1$ for $i\le N$ and $w_{N+1}(x,y)=w(x)$.
Then
\[
\mathbb{P}_{\text{test}}\{Y_{N+1}\in \hat C(X_{N+1})\}\ge 1-\alpha,
\]
where the probability is over $(X_{N+1},Y_{N+1})\sim P^{\text{test}}$ (and the randomness in the calibration sample).
\end{theorem}

\begin{remark}
The guarantee above is exact when $w$ is known. If $w$ is replaced by an estimate $\hat w$ (e.g., learned from unlabeled test covariates), finite-sample distribution-free coverage is no longer guaranteed and instead depends on the accuracy of $\hat w$ \citep{tibshirani2019}.
\end{remark}

\subsection{Density Ratio Estimation}

In practice, estimating the density ratio under covariate shift is non-trivial. 
Let $p_s$ and $p_t$ denote the source (calibration) and target (test) marginal densities of $X$, respectively, and define the density ratio
\[
w(x)=\frac{p_t(x)}{p_s(x)}.
\]
A common strategy is \emph{classifier-based density-ratio estimation}, where a probabilistic classifier distinguishes source ($D=0$) from target ($D=1$) samples. Let $\eta(x)=\mathbb{P}(D=1\mid X=x)$ and $\pi_s,\pi_t$ denote class priors. Then
\[
w(x)=\frac{\pi_s}{\pi_t}\frac{\eta(x)}{1-\eta(x)}.
\]
Another approach is \emph{kernel density estimation} (KDE), which separately estimates $p_s$ and $p_t$ via
\[
\hat p_s(x)=\frac{1}{n_s}\sum_{i=1}^{n_s} K_h(x-x_i^{(s)}), 
\quad
\hat p_t(x)=\frac{1}{n_t}\sum_{i=1}^{n_t} K_h(x-x_i^{(t)}),
\]
and forms $\hat w(x)=\hat p_t(x)/\hat p_s(x)$.

Classifier-based methods can be sensitive to probability calibration and may produce high-variance importance weights under limited support overlap \citep{shimodaira2000,sugiyama2012}, while KDE-based approaches suffer from the curse of dimensionality in moderate or high dimensions \citep{wand1995,silverman1986}.  We instead employ \emph{Kernel Mean Matching} (KMM), which estimates importance weights by directly matching kernel mean embeddings of the weighted source and target distributions, avoiding explicit density estimation \citep{huang2007}.

\begin{figure*}[ht]
  \centering
  \includegraphics[width=\textwidth]{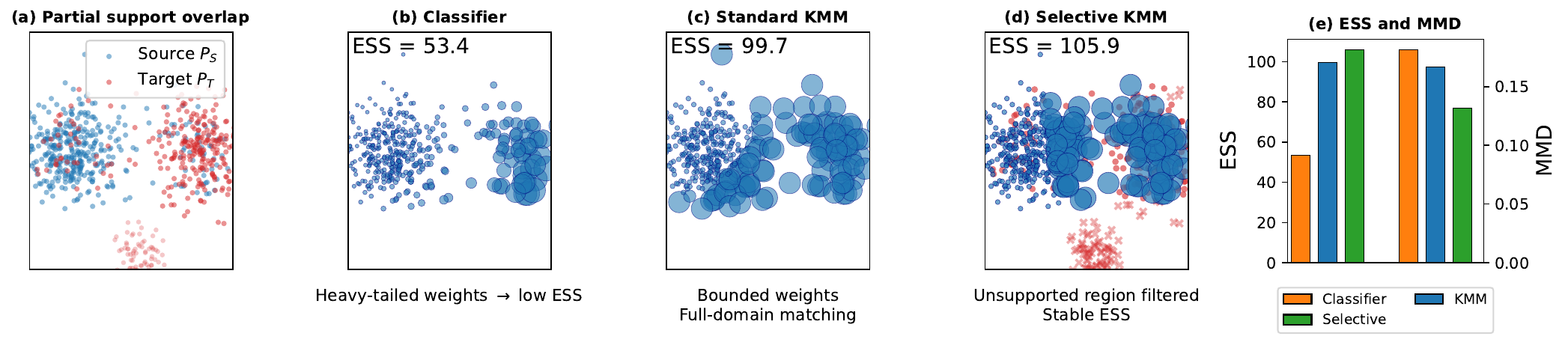}
\caption{Synthetic experiment: (a) Covariate shift with partial support overlap, the target includes regions poorly supported by the source. 
(b) Classifier-based density-ratio assigns heavy-tailed weights to low-support regions, reducing effective sample size (ESS). 
(c) Standard KMM enforces bounded weights and matches moments over the full domain, improving stability but still correcting unsupported regions. 
(d) Selective KMM jointly optimizes source weights and target selection variables, filtering unsupported regions and stabilizing weights. 
(e) Classifier weighting reduces ESS, KMM lowers moment discrepancy (MMD), and selective KMM achieves both stable ESS and improved moment matching.}  
\label{fig:synthetic-covariate-shift}
\end{figure*}

\section{Kernel Mean Matching}
\subsection{Kernel Mean Matching under Covariate Shift}
~\label{sec:kmm}
Let $\{x_i\}_{i=1}^n \sim P_S$ and $\{z_j\}_{j=1}^m \sim P_T$ denote source and target samples, related by a covariate shift. Kernel Mean Matching (KMM) \citep{NIPS2006_a2186aa7} seeks to correct this shift by constructing nonnegative weights $w_1,\dots,w_n$ such that the reweighted empirical source measure
\begin{equation}
P_S^w := \frac{1}{n}\sum_{i=1}^n w_i\,\delta_{x_i}
\end{equation}
approximates (after normalization) $P_T$. Let $\mathcal{H}$ be a Reproducing Kernel Hilbert Space (RKHS) with a characteristic kernel $k$ and feature map $\varphi$. For probability measures $Q$
and $R$ on $\mathbb{R}^d$, define the associated MMD by
\begin{align*}
\mathrm{MMD}(Q,R)
:= \sup_{\|h\|_{\mathcal{H}}\le 1}
\left|
\mathbb{E}_{X\sim Q}[h(X)]
-
\mathbb{E}_{X\sim R}[h(X)]
\right|.
\end{align*}

KMM solves the empirical Maximum Mean Discrepancy (MMD) minimization problem:
\begin{equation}
\min_{w}
\left\|
\frac{1}{n}\sum_{i=1}^n w_i \varphi(x_i)
- \frac{1}{m}\sum_{j=1}^m \varphi(z_j)
\right\|_{\mathcal{H}}^2
\end{equation}
subject to the following for some $\epsilon < 1$
\begin{equation}
0 \le w_i \le B,
\qquad
\left|
\frac{1}{n}\sum_{i=1}^n w_i - 1
\right| \le \epsilon.
\end{equation}



\subsection{Coverage Guarantees under KMM Reweighting}




Let $s:\mathbb{R}^d\times [K]\to \mathbb{R}$ be a conformity score. For $t\in\mathbb{R}$ define the conditional score CDF
\begin{align*}
g_t(x) := \mathbb{P}\bigl(s(x,Y)\le t \,\big|\, X=x\bigr),
\end{align*}
where $(X,Y)\sim P_S$ (equivalently $(X,Y)\sim P_T$) since $P^S_{Y\mid X}=P^T_{Y\mid X}$. The target (population) CDF
of the score is
\begin{align*}
F_T(t)
&:= \mathbb{P}_{(X,Y)\sim P_T}\bigl(s(X,Y)\le t\bigr)
= \mathbb{E}_{X\sim P^T_X}\bigl[g_t(X)\bigr].
\end{align*} Given weights $w_i$ from KMM, define 
$\tilde w_i := \frac{w_i}{\sum_{j=1}^n w_j},
\sum_{i=1}^n \tilde w_i = 1,$
and the corresponding weighted empirical covariate measure
$\widehat P^{\,w}_{S,X} := \sum_{i=1}^n \tilde w_i\,\delta_{x_i}.$
This weighted empirical covariate measure induces a reweighted (population) score CDF 
\begin{align*}
F_w(t)
:= \mathbb{E}_{X\sim \widehat P^{\,w}_{S,X}}\bigl[g_t(X)\bigr]
= \sum_{i=1}^n \tilde w_i\, g_t(x_i).
\end{align*}
The first step toward coverage analysis is to control the population discrepancy $\sup_t |F_w(t)-F_T(t)|$ by the MMD discrepancy between the covariate measures.

\begin{theorem}[\textbf{CDF stability under MMD control}]
\label{thm:cdf-stability}
Assume that for each $t\in\mathbb{R}$ there exists $\tilde g_t\in\mathcal{H}$
such that
\begin{align*}
\sup_{t\in\mathbb{R}} \|\tilde g_t\|_{\mathcal{H}} \le B_{\mathcal{H}},
\qquad
\sup_{t\in\mathbb{R}} \|g_t-\tilde g_t\|_{\infty} \le \varepsilon_{\infty}.
\end{align*}
Then, for any probability measure $Q$ on $\mathbb{R}^d$,
\begin{align*}
\sup_{t\in\mathbb{R}}
\Big|
\mathbb{E}_{X\sim Q}[g_t(X)]
&-
\mathbb{E}_{X\sim P^T_X}[g_t(X)]
\Big| \\
&\le
B_{\mathcal{H}}\,\mathrm{MMD}(Q,P^T_X)
+
2\,\varepsilon_{\infty}.
\end{align*}
In particular, with $Q=\widehat P^{\,w}_{S,X}$ we obtain
\begin{align*}
\sup_{t\in\mathbb{R}} |F_w(t)-F_T(t)|
\le
B_{\mathcal{H}}\,\mathrm{MMD}(\widehat P^{\,w}_{S,X},P^T_X)
+
2\,\varepsilon_{\infty}.
\end{align*}
\end{theorem}
Using this we can show the following.


\begin{theorem}[\textbf{Coverage error for KMM-weighted split conformal}]
\label{thm:coverage-decomposition}
Let $(X_1,Y_1),\dots,(X_n,Y_n)\stackrel{\mathrm{i.i.d.}}{\sim}P_S$ be a calibration sample, let
$Z_1,\dots,Z_m\stackrel{\mathrm{i.i.d.}}{\sim}P_X^T$ denote unlabeled target covariates used to compute the KMM weights, and let $(X_{n+1},Y_{n+1})\sim P_T$
be an independent test point. Define
$\widehat F_w(t) := \sum_{i=1}^n \tilde w_i\,\mathbf{1}\{s(X_i,Y_i)\le t\},$
and $\mathrm{ESS}:=1/\sum_{i=1}^n \tilde w_i^2$.
Let the weighted $(1-\alpha)$-quantile be
$\widehat q_{1-\alpha} := \inf\bigl\{t\in\mathbb{R}:\widehat F_w(t)\ge 1-\alpha\bigr\},$
and define the set-valued predictor
\begin{align*}
\widehat C(x) := \bigl\{y\in[K]: s(x,y)\le \widehat q_{1-\alpha}\bigr\}.
\end{align*}
Assume that with probability $1$ the calibration scores $s(X_1,Y_1),\dots,s(X_n,Y_n)$ are all distinct and that the conditions of Theorem~\ref{thm:cdf-stability} hold with constants $B_{\mathcal{H}}$ and
$\varepsilon_{\infty}$.
Then for any $\delta\in(0,1)$, conditional on $X_1,\dots,X_n,Z_1,\dots,Z_m$, with probability
at least $1-\delta$ over $Y_1,\dots,Y_n$ we have
\begin{align*}
&\Bigl|
\mathbb{P}_{(X,Y)\sim P_T}\bigl(Y\in \widehat C(X)\bigr) - (1-\alpha)
\Bigr| \\
&\!\!\le
B_{\mathcal{H}}\,\mathrm{MMD}_{\mathcal{H}}(\widehat P^{\,w}_{S,X},P_X^T)
+
2\,\varepsilon_{\infty} \! + \!
\left(5+\sqrt{\tfrac{1}{2}\log\tfrac{1}{\delta}}\right)\sqrt{\tfrac{1}{\mathrm{ESS}}}.
\end{align*}
\end{theorem}
\begin{corollary}[Asymptotic validity of KMM-weighted split conformal]
\label{cor:asymptotic-coverage}
Assume the conditions of Theorem~\ref{thm:coverage-decomposition}. Allow the approximation error in
Theorem~\ref{thm:cdf-stability} to depend on $(n,m)$, and write it as $\varepsilon_{\infty,n,m}$.
If, as $n,m\to\infty$,
\begin{align*}
\mathrm{MMD}_{\mathcal{H}}\bigl(\widehat P^{\,w}_{S,X},P_X^T\bigr)\to 0,
\mathrm{ESS}\to\infty,
\varepsilon_{\infty,n,m}\to 0,
\end{align*}
in probability, then the following holds in probability
\begin{align*}
\mathbb{P}_{(X,Y)\sim P_T}\bigl(Y\in \widehat C(X)\bigr)\to 1-\alpha
\end{align*}

\end{corollary}

Theorem~\ref{thm:coverage-decomposition} reveals a clear bias–variance decomposition: the MMD term controls systematic distributional mismatch between the reweighted calibration distribution and the target distribution, while the $1/\sqrt{\mathrm{ESS}}$ term captures variance induced by weight concentration. Accurate covariate-shift correction therefore requires both small moment discrepancy and stable (non-degenerate) weights.

Kernel Mean Matching is well-suited to this objective. Unlike classifier- or KDE-based density-ratio estimation, which can yield heavy-tailed or highly concentrated weights under limited support overlap, KMM minimizes an RKHS discrepancy under explicit box and mass constraints. These constraints control weight variance, preserve effective sample size, and stabilize quantile estimation. As a result, KMM directly targets the two terms governing coverage in our bound—moment alignment and ESS—without explicit density estimation. We further analyze these stability considerations in relation to other density estimation procedures in Appendix ~\ref{app:classifier-instability} and empirically validate them.

\vspace{-0.25 em}
\subsection{Selective Kernel Mean Matching via Joint Target Selection}
\vspace{-0.5 em}
\label{sec:skmm}
When the target distribution $P_T$ is not well-supported by the source distribution $P_S$, the standard density ratio diverges. Forcing the KMM optimizer to match moments over the entire target domain requires extreme source weights, which collapses the Effective Sample Size (ESS) and inflates the variance of our coverage bounds (Theorem~\ref{thm:coverage-decomposition}). 

To address this, we incorporate a selective mechanism that restricts conformal correction to regions of shared support. Specifically, we introduce test-instance selection weights $\alpha_j \in [0,1]$ applied to the target samples $z_j$, alongside the standard source weights $w_i \in [0,B]$ applied to $x_i$. We jointly optimize the source weights and the target selection variables to match the selected target moments. This formulation is related to double-weighting strategies for covariate shift adaptation \cite{segovia2023double}, though our objective and constraints differ and are tailored to conformal calibration rather than risk minimization.
\begin{equation}
\min_{w, \alpha}
\left\|
\frac{1}{n}\sum_{i=1}^n w_i \varphi(x_i)
- \frac{1}{m}\sum_{j=1}^m \alpha_j \varphi(z_j)
\right\|_{\mathcal{H}}^2
\end{equation}
subject to the constraints:
\begin{equation}
\begin{aligned}
0 \le w_i \le B, \qquad
&\left| \frac{1}{n}\sum_{i=1}^n w_i - \frac{1}{m}\sum_{j=1}^m \alpha_j \right| \le \epsilon,\\
0 \le \alpha_j \le 1, \qquad
&\frac{1}{m}\sum_{j=1}^m \alpha_j \ge \tau.
\end{aligned}
\end{equation}
Here, $\tau \in (0,1)$ enforces a acceptance yield (the fraction of target points we wish to make predictions for). The KKT conditions suggest that when gradients clearly favor inclusion or exclusion, many selection variables $\alpha_j$ may approach the boundaries $0$ or $1$ (see Appendix~\ref{app:skmm-theory}). Empirically, we observe such sparsity: the learned $\alpha_j$ often concentrate near the extremes, particularly for target points poorly supported by the source distribution. This induced sparsity directly controls the two terms in the coverage error bound:
\begin{enumerate}
\item \textbf{Bias Reduction (MMD Control):}
By restricting inference to the selected subset, moment matching is concentrated on regions of shared support. In these regions, bounded source weights suffice to approximate the selected target moments, reducing the effective MMD between $P_S^w$ and induced target distribution $\tilde P_T$.

\item \textbf{Variance Reduction ($\mathrm{ESS} \uparrow$):}
Excluding unsupported target points prevents excessive weight concentration on a few samples, yielding a more uniform weight distribution and larger Effective Sample Size (ESS), thereby stabilizing the $C_2 \sqrt{1/\mathrm{ESS}}$ term.
\end{enumerate}


\vspace{-0.5 em}
\section{Kernel Mean Matching Based Conformal Prediction}
\vspace{-0.5 em}

We describe the practical implementation of split conformal prediction under covariate shift using KMM and its selective extension, focusing on classification. We also discuss class-conditional calibration; details are provided in Appendix~\ref{app:mondrian}.

\paragraph{KMM-Based Split Conformal Prediction.}
Given labeled data $\mathcal{D}$ and unlabeled test covariates ${z_j}_{j=1}^m$, we split $\mathcal{D}$ into training and calibration sets, train a predictor on the training split, and compute conformity scores on the calibration data. KMM is then applied between calibration and test covariates to obtain importance weights, which are used to compute a weighted empirical quantile of the calibration scores.

\paragraph{Selective KMM.}
The double-weighted formulation in~\ref{sec:skmm} induces sparsity in the selection variables $\alpha_j$ via complementary slackness, concentrating mass near the boundaries. In finite samples, the solutions cluster near $0$ and $1$ but are not exactly discrete. We therefore adopt a two-stage procedure: first obtain $\alpha_j$ from the double-weighted objective, then threshold the learned $\alpha$ values to retain sufficiently large-$\alpha$ test points, and finally rerun standard KMM on the retained subset. Coverage applies conditionally to the selected test instances under the assumptions of Section~\ref{sec:kmm}. The unified procedure is summarized in Algorithm~\ref{alg:kmm_unified}.

\begin{algorithm}[t]
\caption{Split Conformal Prediction with (Selective) KMM}
\label{alg:kmm_unified}
\begin{algorithmic}[1]
\REQUIRE Labeled data $\mathcal{D}=\{(x_i,y_i)\}_{i=1}^N$, 
test covariates $\{z_j\}_{j=1}^m$, 
miscoverage level $\alpha_{\mathrm{conf}}$, 
kernel $k$, weight bound $B$, threshold $\theta$ (optional)

\ENSURE Prediction sets $\hat C(z_j)$

\STATE Split $\mathcal{D}$ into training and calibration sets $\mathcal{D}_{\mathrm{cal}}=\{(x_i,y_i)\}_{i=1}^n$
\STATE Train predictor $f$ and define conformity score $S(x,y)$
\STATE Compute calibration scores $S_i = S(x_i,y_i)$

\IF{Selective KMM}
    \STATE Solve double-weighted KMM to obtain $\alpha_j$
    \STATE $\mathcal{T}_{\mathrm{sel}} \leftarrow \{z_j : \alpha_j \ge \theta\}$ \quad (or top $\tau m$)
\ELSE
    \STATE $\mathcal{T}_{\mathrm{sel}} \leftarrow \{z_j\}_{j=1}^m$
\ENDIF

\STATE Solve KMM between $\mathcal{D}_{\mathrm{cal}}$ and $\mathcal{T}_{\mathrm{sel}}$ to obtain $w_i$
\STATE Normalize $\tilde w_i = w_i / \sum_{k=1}^n w_k$

\STATE Compute weighted quantile $\hat q_{1-\alpha_{\mathrm{conf}}}$ satisfying
\[
\sum_{i=1}^n \tilde w_i \mathbf{1}\{S_i \le t\} \ge 1-\alpha_{\mathrm{conf}}
\]

\FOR{$z_j \in \mathcal{T}_{\mathrm{sel}}$}
    \STATE $\hat C(z_j) = \{y : S(z_j,y) \le \hat q_{1-\alpha_{\mathrm{conf}}}\}$
\ENDFOR
\end{algorithmic}
\end{algorithm}

\paragraph{Connection to Selective Conformal Inference.}
Our formulation is conceptually related to selective conformal inference \cite{angelopoulos2021learn}, which guarantees coverage on a data-dependent subset of test instances. Classical selective methods typically assume exchangeability and use predictive confidence to define this subset. In contrast, our selection mechanism is driven by support mismatch under covariate shift: guarantees are restricted to regions with stable reweighting and reliable density-ratio estimation.

Selective guarantees are particularly relevant in high-stakes domains. In molecular discovery, conformal prediction is used to prioritize active compounds for experimental validation, where missing truly active molecules is costly; abstaining in low-support regions is therefore practical. Similarly, in healthcare, models may defer out-of-distribution cases to clinicians while maintaining calibrated guarantees on well-supported patient groups.

\vspace{-0.75 em}
\section{Experimental Details}
\vspace{-0.5 em}

In this section, we evaluate whether the proposed methods correct covariate shift and restore conformal coverage on small-molecule property prediction tasks, where conformal prediction is widely used to prioritize compounds for wet-lab validation \cite{astigarraga2025conformal}. These benchmarks derive from real-world drug discovery pipelines \cite{huang2021therapeutics} and are standard testbeds for molecular uncertainty quantification \cite{laghuvarapu2023codrug}.

\vspace{-0.5 em}
\subsection{Datasets}
\vspace{-0.5 em}

We conduct experiments on widely evaluate molecular property prediction benchmarks \cite{fooladi2025evaluating} from Therapeutics Data Commons (TDC) \cite{huang2021therapeutics}.  We evaluate on five classification tasks: \textbf{Tox21} (multi-target toxicity from high-throughput assays), \textbf{AMES} (mutagenicity prediction), \textbf{hERG} (cardiac toxicity via potassium channel inhibition), \textbf{BBB-Martins} (blood--brain barrier penetration), and \textbf{HIV} (bioactivity via inhibition of HIV replication).

These tasks span toxicity and bioactivity prediction and exhibit substantial structural diversity across molecular scaffolds. In such regimes, reliable uncertainty quantification is critical: overconfident predictions on novel compounds can lead to unsafe or inefficient decisions, while overly conservative predictions hinder prioritization. Conformal prediction offers a principled mechanism for controlling predictive uncertainty, making these benchmarks well-suited for evaluating robustness under covariate shift.

\begin{figure*}[t]
  \centering
  \includegraphics[width=0.95\textwidth]{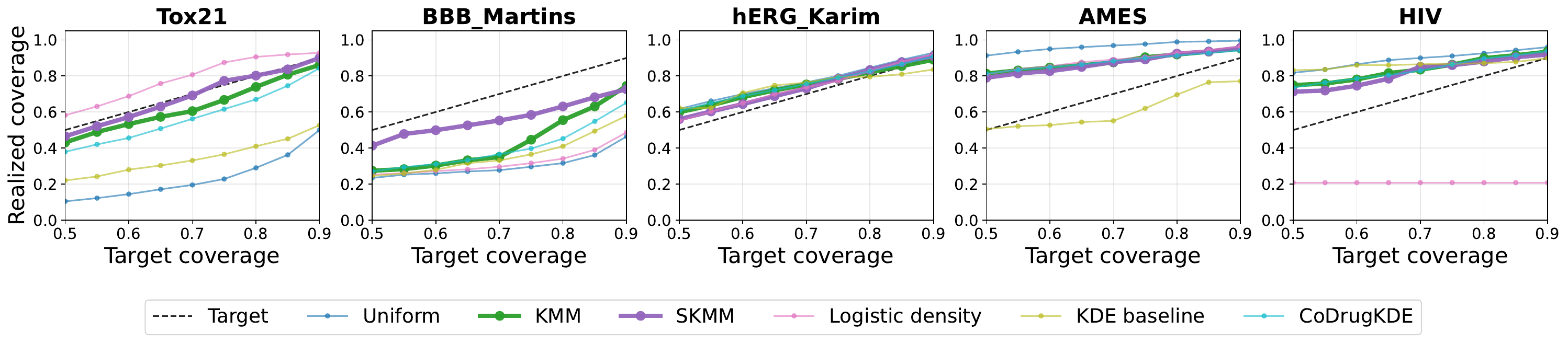}
\caption{Calibration curves under covariate shift (global calibration). The dashed line denotes nominal coverage. Uniform CP undercovers systematically, density-ratio baselines partially correct this, KMM improves alignment, and SKMM most closely tracks the target across datasets.}
\label{fig:calibration-coverage-mondrian}
\vspace{-1em}
\end{figure*}

\subsection{Baselines and Methods}

We compare against the following approaches:
\begin{enumerate}[leftmargin=*, itemsep=1pt, topsep=1pt]
    \item \textbf{Vanilla Conformal Prediction.} Standard split conformal prediction assuming exchangeability.
    \item \textbf{KDE Density Ratio Estimation (64D).} Importance weights via kernel density estimation in 64-dimensional learned feature space.
    \item \textbf{KDE Density Ratio Estimation (8D).} KDE using an 8-dimensional linear projection to assess sensitivity to dimensionality, as suggest by CoDrug\cite{laghuvarapu2023codrug}. We refer to this as \textit{CoDrugKDE}.
    \item \textbf{Logistic Classifier-Based Density Ratio Estimation.} Logistic regression domain classifier distinguishing calibration and test samples.
    \item \textbf{KMM}: Kernel moment matching with bounded importance weights but no selective filtering.
    \item \textbf{SKMM (KMM + $\alpha$ Filtering)} Joint optimization of calibration weights ($\beta$) and target selection variables ($\alpha$), corresponding to our selective KMM framework.
\end{enumerate}

\vspace{-0.25 em}
\subsection{Implementation Details}
\vspace{-0.25 em}

\paragraph{Covariate Shift Construction.}
To simulate realistic structural shift, we consider Fingerprint split \cite{wu2018moleculenet}, following prior work \cite{laghuvarapu2023codrug}. In this split, molecules are represented using binary embeddings (fingerprints) and partitioned based on distance from the dataset centroid in fingerprint space. The 15\% farthest molecules form the test set, while the remainder constitutes the calibration pool. This induces reduced scaffold overlap and support mismatch, reflecting practical discovery settings where new chemical series differ from historical data.

\paragraph{Model Training.}
All predictive models use the AttentiveFP graph neural network architecture~\cite{attentivefp}. Models are trained with five independent random seeds per dataset. Conformal calibration and evaluation are performed separately for each run, and results are averaged.

\paragraph{Weight Construction.}
For weighting-based methods, covariate representations are extracted from the final hidden layer of AttentiveFP (64-dimensional embeddings). These embeddings serve as inputs for density ratio estimation and kernel moment matching, reflecting realistic high-dimensional representation learning regimes. For the \textit{CoDrugKDE} baseline, we evaluate an 8-dimensional linear projection to assess sensitivity to dimensionality.

\paragraph{Conformal Setup.}
Equation~\ref{eq:valid-coverage} describes the standard marginal coverage guarantee of conformal prediction, where coverage holds over the joint distribution of $(X,Y)$. In addition to global calibration, we also evaluate \emph{class-conditional (Mondrian) conformal prediction}, which computes quantile thresholds separately within predefined groups—in our case, class labels. Mondrian conformal prediction guarantees coverage conditional on class membership, i.e.,
$\mathbb{P}\!\left(Y_{N+1} \in \hat C(X_{N+1}) \mid Y_{N+1}=c\right) \ge 1-\alpha$
for each class $c$. For molecular activity tasks, we calibrate on the active class (C1), while for toxicity tasks we calibrate on the non-toxic class (C0), prioritizing reliable coverage for outcomes most critical to downstream drug discovery decisions. Additional details of this are provided in Appendix~\ref{app:mondrian}.

\paragraph{KMM Hyperparameters.}
For kernel mean matching (KMM), the kernel bandwidth is selected using the maximum test power criterion following prior work ~\cite{gretton2012kernel}. The weight upper bound is set to $B = 30$, and the target selection fraction is fixed at $\alpha = 0.5$, corresponding to selecting 50\% of the test molecules. These hyperparameters were not tuned per dataset. 

Additional implementation details and hyperparameters are provided in Appendix~\ref{app:exp}. All reported results are averaged over five independent random runs and detailed error bars are included in Appendix~\ref{app:results}.

\section{Results and Discussion}

\subsection{Coverage Under Covariate Shift}
We first evaluate vanilla (uniform) conformal prediction under both no-shift and fingerprint-based covariate shift. Results are summarized in Table~\ref{tab:uniform_shift}. To quantify calibration error, we report the mean absolute deviation (MAD) from nominal coverage. Specifically, for target coverage levels $1-\alpha \in \{0.5, 0.55, \dots, 0.9\}$,
\[
\mathrm{MAD} = \frac{1}{K} \sum_{k=1}^{K} 
\left| \widehat{\mathrm{cov}}(1-\alpha_k) - (1-\alpha_k) \right|,
\]
where $\widehat{\mathrm{cov}}(1-\alpha_k)$ denotes empirical coverage at level $1-\alpha_k$, and $K$ is the number of evaluated levels. This metric aggregates deviation across coverage levels and provides a measure of miscalibration.

Under random (no-shift) splits, uniform conformal prediction achieves low MAD (e.g., 0.0044 on HIV and 0.0121 on AMES), confirming nominal coverage under exchangeability. Under fingerprint-based covariate shift, calibration deteriorates sharply: MAD exceeds 0.43 on Tox21 and 0.45 on BBB-Martins, and increases by more than an order of magnitude even on milder shifts such as HIV. These results demonstrate that exchangeability violations substantially degrade conformal validity in realistic molecular settings, motivating explicit covariate-shift correction.
\begin{table}[t]
\centering
\caption{Mean absolute deviation (MAD) of uniform conformal prediction under random (no-shift) and shifted splits.}
\label{tab:uniform_shift}
\small
\begin{tabular}{lcc}
\toprule
Dataset & Random & Shifted \\
\midrule
AMES          & \textbf{0.0121} & 0.0844 \\
BBB-Martins   & \textbf{0.0346} & 0.4579 \\
HIV           & \textbf{0.0044} & 0.1830 \\
Tox21         & \textbf{0.0249} & 0.4366 \\
hERG-Karim    & \textbf{0.0254} & 0.0384 \\
\bottomrule
\end{tabular}
\vspace{-1em}
\end{table}

\begin{table*}[t]
\centering
\caption{Mean absolute deviation (MAD) from nominal coverage under covariate shift for global (G) and Mondrian (M) calibration. Lower is better. Best per column is \textbf{bold}; second-best is \textcolor{gray}{\textbf{gray bold}}. SKMM attains the lowest MAD under global calibration and the best overall mean rank across both settings.}
\label{tab:mad}
\small
\setlength{\tabcolsep}{4pt}
\begin{tabular}{l|cc|cc|cc|cc|cc|cc}
\toprule
 & \multicolumn{2}{c|}{Uniform} 
 & \multicolumn{2}{c|}{KDE} 
 & \multicolumn{2}{c|}{CoDrug KDE} 
 & \multicolumn{2}{c|}{Logistic} 
 & \multicolumn{2}{c|}{KMM} 
 & \multicolumn{2}{c}{SKMM} \\
Dataset 
 & G & M 
 & G & M 
 & G & M 
 & G & M 
 & G & M 
 & G & M \\
\midrule
AMES 
 & 0.166 & 0.084 
 & 0.137 & \textbf{0.046} 
 & 0.093 & 0.090 
 & 0.102 & 0.096 
 & \textcolor{gray}{\textbf{0.078}} & 0.079 
 & \textbf{0.065} & \textcolor{gray}{\textbf{0.072}} \\

Tox21 
 & 0.409 & 0.437 
 & 0.323 & 0.331 
 & 0.103 & 0.110 
 & 0.073 & 0.089 
 & \textcolor{gray}{\textbf{0.053}} & \textcolor{gray}{\textbf{0.057}} 
 & \textbf{0.015} & \textbf{0.020} \\

BBB-Martins 
 & 0.424 & 0.458 
 & 0.364 & 0.393 
 & 0.329 & 0.381 
 & 0.396 & 0.443 
 & \textcolor{gray}{\textbf{0.319}} & \textcolor{gray}{\textbf{0.353}} 
 & \textbf{0.213} & \textbf{0.270} \\

hERG 
 & 0.034 & 0.038 
 & 0.096 & 0.094 
 & 0.031 & 0.030 
 & \textcolor{gray}{\textbf{0.020}} & \textcolor{gray}{\textbf{0.022}} 
 & 0.028 & 0.036 
 & \textbf{0.009} & \textbf{0.019} \\

HIV 
 & 0.185 & 0.183 
 & 0.291 & 0.283 
 & 0.145 & 0.151 
 & 0.297 & 0.269 
 & \textcolor{gray}{\textbf{0.087}} & \textcolor{gray}{\textbf{0.098}} 
 & \textbf{0.013} & \textbf{0.022} \\

\midrule
Mean Rank 
 & 5.4 & 5.0 
 & 5.0 & 4.4 
 & 3.4 & 3.6 
 & 4.0 & 4.2 
 & \textcolor{gray}{\textbf{2.2}} & \textcolor{gray}{\textbf{2.6}} 
 & \textbf{1.0} & \textbf{1.2} \\
\bottomrule
\end{tabular}
\end{table*}
\vspace{-1em}

\subsection{Global and Mondrian Calibration}
\vspace{-0.25em}

Table~\ref{tab:mad} reports mean absolute deviation (MAD) from nominal coverage under both global and Mondrian conformal calibration. Uniform conformal prediction performs poorly under covariate shift, exhibiting the highest deviation across nearly all datasets, with mean ranks of 5.4 (global) and 5.0 (Mondrian). In particular, uniform MAD exceeds 0.40 on Tox21 and BBB-Martins, confirming that exchangeability violations severely degrade calibration.

Density-ratio baselines partially mitigate this effect but remain inconsistent. Logistic weighting substantially improves Tox21 (0.073 vs.\ 0.409 under uniform), yet remains unstable on BBB-Martins (0.396) and HIV (0.297). KDE-based approaches are particularly sensitive to the dimensionality of the representation, with MAD exceeding 0.30 on multiple datasets under global calibration. Even the CoDrug KDE variant remains above 0.30 in BBB-Martins and fails to consistently match nominal coverage.

In contrast, KMM (no filtering) consistently reduces deviation across datasets. On Tox21 (global), MAD decreases from 0.073 (logistic) to 0.053 under KMM, with similar improvements in BBB-Martins (0.319) and HIV (0.087), indicating more stable moment alignment under covariate shift. The selective extension (SKMM) further improves calibration, achieving the lowest MAD on all five benchmarks under global calibration—reducing MAD to 0.015 in Tox21, 0.009 in hERG, and 0.013 in HIV. Under Mondrian calibration, SKMM remains competitive and attains the best mean rank (1.2), showing that its gains extend to class-conditional settings.

Furthermore, figure~\ref{fig:calibration-coverage-mondrian} corroborates these observations: SKMM consistently tracks the diagonal target line across coverage levels, while uniform and KDE baselines exhibit systematic undercoverage under shift. KMM narrows this gap, and SKMM further stabilizes calibration across both toxicity and activity prediction tasks. We observe similar trends global calibration, depicted in Appendix ~\ref{app:exp}.

Overall, the results demonstrate that moment-matching-based reweighting improves calibration stability under covariate shift, and that selective filtering further enhances robustness in low-overlap regimes. On average across datasets under global calibration, SKMM achieves a 55\% reduction in MAD relative to CoDrugKDE, highlighting substantial improvements in coverage stability.

\subsection{Bias--Variance and Selective Filtering}
\label{sec:bias-variance}
\begin{figure}[ht]
  \centering
  \includegraphics[width=0.7\linewidth]{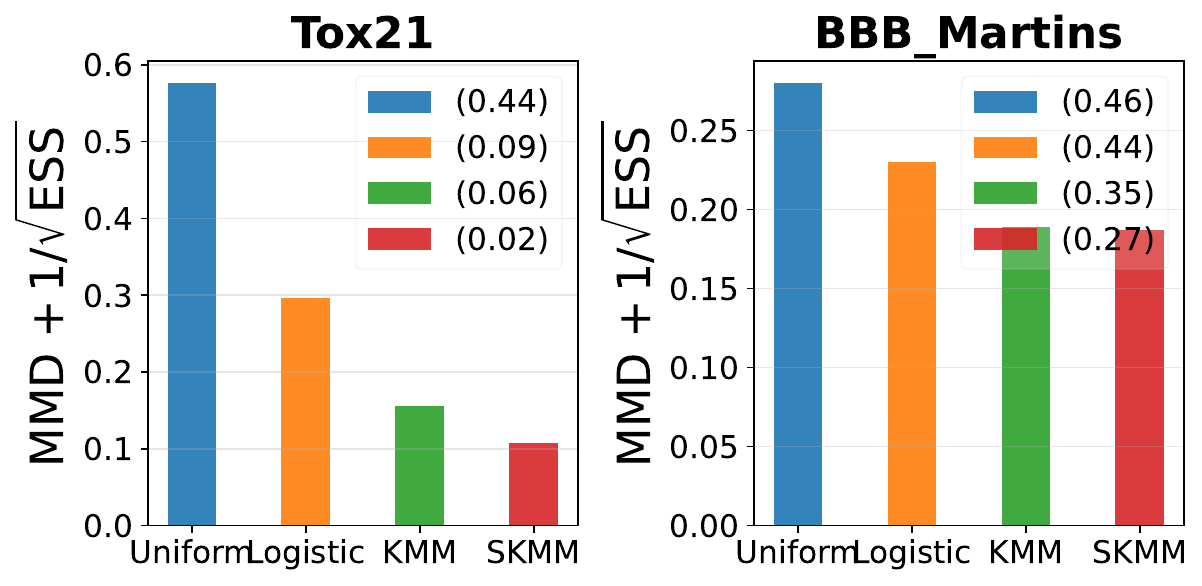}
  \caption{Bias--variance proxy ($\mathrm{MMD} + \sqrt{1/\mathrm{ESS}}$) under covariate shift. MAD in brackets.}
  \label{fig:full-shift-mmd-ess}
\end{figure}
\vspace{-1em}
\begin{figure}[ht]
  \centering
  \includegraphics[width=0.7\linewidth]{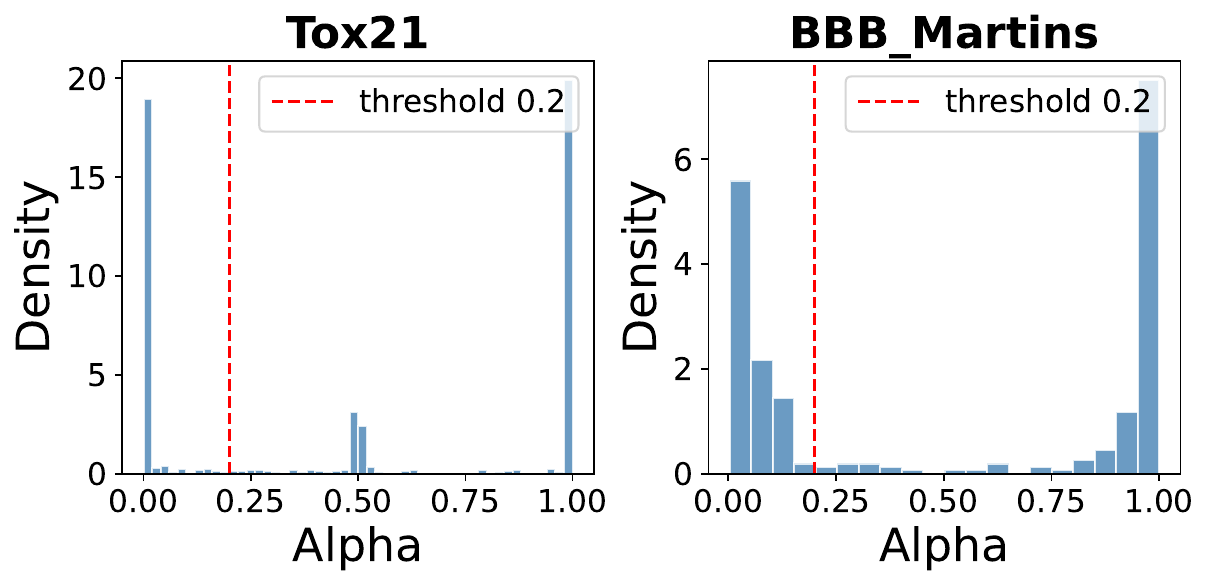}
  \caption{Selection weights under SKMM.}
  \label{fig:alphas}
\end{figure}

To understand the source of improvement under covariate shift, we examine the dominant terms in our coverage bound, namely $
\mathrm{MMD} + \sqrt{1/\mathrm{ESS}} $, which capture bias (distributional mismatch) and variance (weight concentration), respectively. Figure~\ref{fig:full-shift-mmd-ess} reports this bias--variance proxy across methods. Uniform conformal prediction exhibits large combined error due to both high MMD and lack of shift correction. The ordering of $\mathrm{MMD} + \sqrt{1/\mathrm{ESS}}$ across methods closely aligns with empirical calibration performance, with KMM and SKMM achieving the smallest values

Selective KMM (SKMM) further reduces this quantity by restricting correction to regions of shared support, achieving the smallest combined error across datasets. To better understand this behavior, Figure~\ref{fig:alphas} visualizes the learned selection weights $\alpha$. The distribution of $\alpha$ concentrates near the boundaries, with approximately half of the test instances retained (as controlled by $\tau = 0.5$). The learned sparsity is consistent across datasets and remains stable under shift.
and exhibit the same qualitative trends.
\section{Conclusion}

We presented KMM-CP, a framework for conformal prediction under covariate shift that leverages Kernel Mean Matching with a selective extension for limited-support regimes. By combining bounded moment matching with target-aware filtering, the proposed method improves coverage stability relative to standard density-ratio approaches, particularly in high-dimensional molecular representation spaces. Our analysis clarifies how moment alignment and effective sample size govern coverage under shift, and experiments across multiple molecular property benchmarks demonstrate more reliable coverage under structural distribution shift.

We acknowledge several limitations. The selective filtering step reduces the number of evaluated samples, and performance depends on the quality of learned representations and kernel hyperparameters. Although our synthetic shift construction reflects realistic structural mismatch, real-world deployment may involve more complex distribution shifts. Finally, while we demonstrate effectiveness on small-molecule property prediction tasks, generalization to other domains and data modalities remains to be validated.





\bibliography{uai2026-template}

\newpage

\onecolumn

\title{Appendix}
\appendix

\appendix

\section{Proofs of Theoretical Results}

\subsection{Proof of Theorem~\ref{thm:cdf-stability}}

Fix a probability measure $Q$ on $\mathbb{R}^d$ and define
\begin{align*}
F_Q(t) := \mathbb{E}_{X\sim Q}[g_t(X)].
\end{align*}
Then, for any fixed $t\in\mathbb{R}$,
\begin{align*}
F_Q(t)-F_T(t)
=
\mathbb{E}_{X\sim Q}[g_t(X)]
-
\mathbb{E}_{X\sim P_X^T}[g_t(X)].
\end{align*}

Add and subtract $\tilde g_t$ to obtain
\begin{align*}
\bigl|F_Q(t)-F_T(t)\bigr|
&\le
\left|
\mathbb{E}_{X\sim Q}[\tilde g_t(X)]
-
\mathbb{E}_{X\sim P_X^T}[\tilde g_t(X)]
\right| \\
&\quad +
\left|
\mathbb{E}_{X\sim Q}[g_t(X)-\tilde g_t(X)]
\right|
+
\left|
\mathbb{E}_{X\sim P_X^T}[g_t(X)-\tilde g_t(X)]
\right|.
\end{align*}

Since $|g_t(x)-\tilde g_t(x)|\le \|g_t-\tilde g_t\|_\infty$ for all $x$, and both $Q$ and $P_X^T$ are probability measures, we have
\begin{align*}
\left|
\mathbb{E}_{X\sim Q}[g_t(X)-\tilde g_t(X)]
\right|
\le
\|g_t-\tilde g_t\|_\infty,
\qquad
\left|
\mathbb{E}_{X\sim P_X^T}[g_t(X)-\tilde g_t(X)]
\right|
\le
\|g_t-\tilde g_t\|_\infty.
\end{align*}

For the first term, if $\|\tilde g_t\|_{\mathcal{H}}=0$, then $\tilde g_t=0$ in $\mathcal{H}$ and hence
\begin{align*}
\left|
\mathbb{E}_{X\sim Q}[\tilde g_t(X)]
-
\mathbb{E}_{X\sim P_X^T}[\tilde g_t(X)]
\right| = 0.
\end{align*}
Otherwise, define
\begin{align*}
h_t := \frac{\tilde g_t}{\|\tilde g_t\|_{\mathcal{H}}},
\end{align*}
so that $\|h_t\|_{\mathcal{H}}=1$. By the definition of $\mathrm{MMD}_{\mathcal{H}}$,
\begin{align*}
\left|
\mathbb{E}_{X\sim Q}[\tilde g_t(X)]
-
\mathbb{E}_{X\sim P_X^T}[\tilde g_t(X)]
\right|
&=
\|\tilde g_t\|_{\mathcal{H}}
\left|
\mathbb{E}_{X\sim Q}[h_t(X)]
-
\mathbb{E}_{X\sim P_X^T}[h_t(X)]
\right| \\
&\le
\|\tilde g_t\|_{\mathcal{H}}\,\mathrm{MMD}_{\mathcal{H}}(Q,P_X^T).
\end{align*}

Combining the bounds gives
\begin{align*}
|F_Q(t)-F_T(t)|
\le
\|\tilde g_t\|_{\mathcal{H}}\,\mathrm{MMD}_{\mathcal{H}}(Q,P_X^T)
+
2\|g_t-\tilde g_t\|_\infty.
\end{align*}
Taking the supremum over $t\in\mathbb{R}$ and using the uniform bounds
\begin{align*}
\sup_{t\in\mathbb{R}}\|\tilde g_t\|_{\mathcal{H}} \le B_{\mathcal{H}},
\qquad
\sup_{t\in\mathbb{R}}\|g_t-\tilde g_t\|_\infty \le \varepsilon_\infty,
\end{align*}
we obtain
\begin{align*}
\sup_{t\in\mathbb{R}} |F_Q(t)-F_T(t)|
\le
B_{\mathcal{H}}\,\mathrm{MMD}_{\mathcal{H}}(Q,P_X^T)
+
2\varepsilon_\infty.
\end{align*}

Finally, applying this bound with $Q=\widehat P^{\,w}_{S,X}$ and recalling that
\begin{align*}
F_w(t)=\mathbb{E}_{X\sim \widehat P^{\,w}_{S,X}}[g_t(X)],
\end{align*}
yields
\begin{align*}
\sup_{t\in\mathbb{R}} |F_w(t)-F_T(t)|
\le
B_{\mathcal{H}}\,\mathrm{MMD}_{\mathcal{H}}(\widehat P^{\,w}_{S,X},P_X^T)
+
2\varepsilon_\infty.
\end{align*}
This proves the theorem.













\subsection{Proof of Theorem~\ref{thm:coverage-decomposition}}

\begin{proof}
    Throughout the proof, condition on $X_1,\dots,X_n,Z_1,\dots,Z_m$ so the weights $\tilde w_1,\dots,\tilde w_n$ are deterministic and the labels $Y_1,\dots,Y_n$ remain independent.
    \begin{enumerate}
        \item \textbf{Reduce coverage error to a CDF error and an overshoot term.}
        By definition of $\widehat C$,
        \begin{align*}
        \mathbb{P}_{(X,Y)\sim P_T}\bigl(Y\in \widehat C(X)\bigr)
        =
        \mathbb{P}_{(X,Y)\sim P_T}\bigl(s(X,Y)\le \widehat q_{1-\alpha}\bigr)
        =
        F_T(\widehat q_{1-\alpha}).
        \end{align*}
        Since $\widehat F_w(\widehat q_{1-\alpha})\ge 1-\alpha$ by definition of $\widehat q_{1-\alpha}$,
        \begin{align*}
        F_T(\widehat q_{1-\alpha})-(1-\alpha)
        &=
        \bigl(F_T(\widehat q_{1-\alpha})-\widehat F_w(\widehat q_{1-\alpha})\bigr)
        +
        \bigl(\widehat F_w(\widehat q_{1-\alpha})-(1-\alpha)\bigr),
        \end{align*}
        and therefore
        \begin{align*}
        \Bigl|F_T(\widehat q_{1-\alpha})-(1-\alpha)\Bigr|
        \le
        \sup_{t\in\mathbb{R}} |F_T(t)-\widehat F_w(t)|
        +
        \bigl(\widehat F_w(\widehat q_{1-\alpha})-(1-\alpha)\bigr).
        \end{align*}
        Because $\widehat F_w$ is right-continuous and has jumps of size exactly $\tilde w_i$ at the distinct points
        $s(X_i,Y_i)$, the distinct-scores assumption implies that every jump has size at most $\max_i \tilde w_i$.
        Moreover, by minimality of $\widehat q_{1-\alpha}$ we have $\widehat F_w(t)<1-\alpha$ for all $t<\widehat q_{1-\alpha}$,
        hence the overshoot satisfies
        \begin{align*}
        0 \le \widehat F_w(\widehat q_{1-\alpha})-(1-\alpha) \le \max_{1\le i\le n}\tilde w_i.
        \end{align*}
        Therefore
        \begin{align*}
        \Bigl|
        \mathbb{P}_{(X,Y)\sim P_T}\bigl(Y\in \widehat C(X)\bigr) - (1-\alpha)
        \Bigr|
        \le
        \sup_{t\in\mathbb{R}} |F_T(t)-\widehat F_w(t)|
        +
        \max_{1\le i\le n}\tilde w_i.
        \end{align*}
        \item \textbf{Decompose $\sup_t|F_T(t)-\widehat F_w(t)|$ into bias and variance.}
        By the triangle inequality,
        \begin{align*}
        \sup_{t\in\mathbb{R}} |F_T(t)-\widehat F_w(t)|
        \le
        \sup_{t\in\mathbb{R}} |F_T(t)-F_w(t)|
        +
        \sup_{t\in\mathbb{R}} |F_w(t)-\widehat F_w(t)|.
        \end{align*}
        The first term is controlled by Theorem~\ref{thm:cdf-stability}:
        \begin{align*}
        \sup_{t\in\mathbb{R}} |F_T(t)-F_w(t)|
        \le
        B_{\mathcal{H}}\,\mathrm{MMD}_{\mathcal{H}}(\widehat P^{\,w}_{S,X},P_X^T)
        +
        2\,\varepsilon_{\infty}.
        \end{align*}
        It remains to control $\sup_t |F_w(t)-\widehat F_w(t)|$.
        \item \textbf{An expectation bound for the weighted empirical process.}
        Define
        \begin{align*}
        \Delta := \sup_{t\in\mathbb{R}} \left|\widehat F_w(t)-F_w(t)\right|
        =
        \sup_{t\in\mathbb{R}}
        \left|
        \sum_{i=1}^n \tilde w_i
        \Bigl(\mathbf{1}\{s(X_i,Y_i)\le t\}-g_t(X_i)\Bigr)
        \right|.
        \end{align*}
        Let $Y_1',\dots,Y_n'$ be an independent copy of $Y_1,\dots,Y_n$ conditional on $X_1,\dots,X_n$ (and thus also on the
        weights). Define
        \begin{align*}
        \widehat F_w'(t) := \sum_{i=1}^n \tilde w_i\,\mathbf{1}\{s(X_i,Y_i')\le t\}.
        \end{align*}
        By Jensen's inequality and $F_w(t)=\mathbb{E}[\widehat F_w(t)\mid X_1,\dots,X_n,Z_1,\dots,Z_m]$,
        \begin{align*}
        \mathbb{E}\bigl[\Delta \,\big|\, X_1,\dots,X_n,Z_1,\dots,Z_m\bigr]
        &\le
        \mathbb{E}\left[
        \sup_t \left|\widehat F_w(t)-\widehat F_w'(t)\right|
        \,\Big|\, X_1,\dots,X_n,Z_1,\dots,Z_m
        \right].
        \end{align*}
        Introduce i.i.d. Rademacher signs $\varepsilon_1,\dots,\varepsilon_n$ independent of everything else. By standard
        symmetrization,
        \begin{align*}
        \mathbb{E}\bigl[\Delta \,\big|\, X_1,\dots,X_n,Z_1,\dots,Z_m\bigr]
        \le
        2\,\mathbb{E}\left[
        \sup_{t\in\mathbb{R}}
        \left|
        \sum_{i=1}^n \tilde w_i\,\varepsilon_i\,\mathbf{1}\{s(X_i,Y_i)\le t\}
        \right|
        \,\Big|\, X_1,\dots,X_n,Z_1,\dots,Z_m
        \right].
        \end{align*}
        Condition further on $X_1,\dots,X_n,Z_1,\dots,Z_m,Y_1,\dots,Y_n$ and write $s_i:=s(X_i,Y_i)$.
        Let $\pi$ be a permutation such that $s_{\pi(1)}\le \cdots \le s_{\pi(n)}$. Then
        \begin{align*}
        \sup_{t\in\mathbb{R}}
        \left|
        \sum_{i=1}^n \tilde w_i\,\varepsilon_i\,\mathbf{1}\{s_i\le t\}
        \right|
        =
        \max_{0\le k\le n}
        \left|
        \sum_{j=1}^k \tilde w_{\pi(j)}\,\varepsilon_{\pi(j)}
        \right|.
        \end{align*}
        Define $M_k:=\sum_{j=1}^k \tilde w_{\pi(j)}\,\varepsilon_{\pi(j)}$ with $M_0:=0$. By Doob's $L^2$ maximal inequality,
        \begin{align*}
        \mathbb{E}\left[\max_{0\le k\le n} M_k^2 \,\Big|\, X_1,\dots,X_n,Z_1,\dots,Z_m,Y\right]
        \le
        4\,\mathbb{E}\left[M_n^2 \,\Big|\, X_1,\dots,X_n,Z_1,\dots,Z_m,Y\right]
        =
        4\sum_{i=1}^n \tilde w_i^2.
        \end{align*}
        Taking square roots and using Jensen's inequality yields
        \begin{align*}
        \mathbb{E}\left[
        \sup_{t\in\mathbb{R}}
        \left|
        \sum_{i=1}^n \tilde w_i\,\varepsilon_i\,\mathbf{1}\{s_i\le t\}
        \right|
        \,\Big|\, X_1,\dots,X_n,Z_1,\dots,Z_m,Y
        \right]
        \le
        2\sqrt{\sum_{i=1}^n \tilde w_i^2}.
        \end{align*}
        Combining the last two displays and unconditioning on $Y$ gives
        \begin{align*}
        \mathbb{E}\bigl[\Delta \,\big|\, X_1,\dots,X_n,Z_1,\dots,Z_m\bigr]
        \le
        4\sqrt{\sum_{i=1}^n \tilde w_i^2}
        =
        4\sqrt{\tfrac{1}{\mathrm{ESS}}}.
        \end{align*}

        \item \textbf{A high-probability bound via bounded differences.}
        View $\Delta$ as a function of the independent variables $Y_1,\dots,Y_n$ conditional on $X_1,\dots,X_n,Z_1,\dots,Z_m$.
        If one replaces $Y_i$ by another value $Y_i^{\sharp}$ (leaving all other $Y_j$ fixed), then for every $t$ the quantity
        $\widehat F_w(t)$ changes by at most $\tilde w_i$, and $F_w(t)$ is unchanged. Hence $\Delta$ changes by at most $\tilde w_i$.
        McDiarmid's inequality implies that for any $u>0$,
        \begin{align*}
        \mathbb{P}\bigl(\Delta-\mathbb{E}[\Delta\mid X,Z]\ge u \,\big|\, X,Z\bigr)
        \le
        \exp\left(-\frac{2u^2}{\sum_{i=1}^n \tilde w_i^2}\right)
        =
        \exp\bigl(-2u^2\,\mathrm{ESS}\bigr),
        \end{align*}
        where $X$ denotes $(X_1,\dots,X_n)$ and $Z$ denotes $(Z_1,\dots,Z_m)$.
        Taking $u:=\sqrt{\tfrac{1}{2\mathrm{ESS}}\log\tfrac{1}{\delta}}$ yields that with probability at least $1-\delta$,
        \begin{align*}
        \Delta
        \le
        4\sqrt{\tfrac{1}{\mathrm{ESS}}}
        +
        \sqrt{\tfrac{1}{2\mathrm{ESS}}\log\tfrac{1}{\delta}}.
        \end{align*}

        On the event from item 4, we combine items 1 and 2 to obtain
        \begin{align*}
        \Bigl|
        \mathbb{P}_{(X,Y)\sim P_T}\bigl(Y\in \widehat C(X)\bigr) - (1-\alpha)
        \Bigr|
        &\le
        \sup_t |F_T(t)-F_w(t)|
        +
        \sup_t |F_w(t)-\widehat F_w(t)|
        +
        \max_i \tilde w_i \\
        &\le
        B_{\mathcal{H}}\,\mathrm{MMD}_{\mathcal{H}}(\widehat P^{\,w}_{S,X},P_X^T)
        +
        2\,\varepsilon_{\infty}
        +
        \Delta
        +
        \max_i \tilde w_i.
        \end{align*}
        Finally, $\max_i \tilde w_i \le \sqrt{\sum_{i=1}^n \tilde w_i^2}=\sqrt{1/\mathrm{ESS}}$, so item 4 gives
        \begin{align*}
        \Delta+\max_i\tilde w_i
        \le
        5\sqrt{\tfrac{1}{\mathrm{ESS}}}
        +
        \sqrt{\tfrac{1}{2\mathrm{ESS}}\log\tfrac{1}{\delta}}.
        \end{align*}
        Substituting this completes the proof.
    \end{enumerate}
\end{proof}

\subsection{Proof of Corollary~\ref{cor:asymptotic-coverage}}
\begin{proof}
    Fix $\eta>0$ and $\delta\in(0,1)$. By Theorem~\ref{thm:coverage-decomposition}, conditional on
    $X_1,\dots,X_n,Z_1,\dots,Z_m$, with probability at least $1-\delta$ over $Y_1,\dots,Y_n$,
    \begin{align*}
    \Bigl|
    \mathbb{P}_{(X,Y)\sim P_T}\bigl(Y\in \widehat C(X)\bigr) - (1-\alpha)
    \Bigr|
    &\le
    B_{\mathcal{H}}\,\mathrm{MMD}_{\mathcal{H}}\bigl(\widehat P^{\,w}_{S,X},P_X^T\bigr)
    +
    2\,\varepsilon_{\infty,n,m}
    +
    \left(5+\sqrt{\tfrac{1}{2}\log\tfrac{1}{\delta}}\right)\sqrt{\tfrac{1}{\mathrm{ESS}}}.
    \end{align*}
    Define the random bound
    \begin{align*}
    R_{n,m}(\delta)
    :=
    B_{\mathcal{H}}\,\mathrm{MMD}_{\mathcal{H}}\bigl(\widehat P^{\,w}_{S,X},P_X^T\bigr)
    +
    2\,\varepsilon_{\infty,n,m}
    +
    \left(5+\sqrt{\tfrac{1}{2}\log\tfrac{1}{\delta}}\right)\sqrt{\tfrac{1}{\mathrm{ESS}}}.
    \end{align*}
    Then,
    \begin{align*}
    \mathbb{P}\Bigl(
    \Bigl|
    \mathbb{P}_{(X,Y)\sim P_T}\bigl(Y\in \widehat C(X)\bigr) &- (1-\alpha)
    \Bigr| > \eta
    \Bigr)\\
    &=
    \mathbb{E}\Bigl[
    \mathbb{P}\Bigl(
    \Bigl|
    \mathbb{P}_{(X,Y)\sim P_T}\bigl(Y\in \widehat C(X)\bigr) - (1-\alpha)
    \Bigr| > \eta
    \,\Big|\, X_1,\dots,X_n,Z_1,\dots,Z_m
    \Bigr)
    \Bigr] \\
    &\le
    \mathbb{E}\Bigl[
    \mathbf{1}\{R_{n,m}(\delta)>\eta\}
    +
    \mathbb{P}\bigl(\text{the bound above fails}\,\big|\,X_1,\dots,X_n,Z_1,\dots,Z_m\bigr)
    \Bigr] \\
    &\le
    \mathbb{P}\bigl(R_{n,m}(\delta)>\eta\bigr)+\delta,
    \end{align*}

    where the probability is over $X_1,\dots,X_n,Z_1,\dots,Z_m,Y_1,\dots,Y_n$.
    By assumption, $R_{n,m}(\delta)\to 0$ in probability for each fixed $\delta$, hence
    $\mathbb{P}(R_{n,m}(\delta)>\eta)\to 0$. Taking $\limsup$ and then letting $\delta\downarrow 0$ yields
    \begin{align*}
    \lim_{n,m\to\infty}
    \mathbb{P}\Bigl(
    \Bigl|
    \mathbb{P}_{(X,Y)\sim P_T}\bigl(Y\in \widehat C(X)\bigr) - (1-\alpha)
    \Bigr| > \eta
    \Bigr)
    =0.
    \end{align*}
\end{proof}




\section{Justification for Selective KMM}
\label{app:skmm-theory}

\begin{theorem}[Conditional boundary rule under a yield constraint]
\label{lem:skmm-boundary-rule}
Fix $w$ and consider the optimization over $\alpha\in\mathbb{R}^m$
\begin{align*}
\min_{\alpha}\quad & J(\alpha) \\
\text{s.t.}\quad & 0 \le \alpha_j \le 1 \quad \text{for all } j\in[m], \\
& \frac{1}{m}\sum_{j=1}^m \alpha_j \ge \tau,
\end{align*}
where $J$ is convex and differentiable. Assume $\tau<1$, so Slater's condition holds.
Let $\alpha^\star$ be an optimal solution, and let $(\lambda^\star,\mu^\star,\nu^\star)$ be KKT multipliers, where
$\lambda_j^\star\ge 0$ and $\mu_j^\star\ge 0$ correspond to the constraints $-\alpha_j\le 0$ and $\alpha_j-1\le 0$,
and $\nu^\star\ge 0$ corresponds to the yield constraint $\tau-\frac{1}{m}\sum_{j=1}^m \alpha_j \le 0$.
Then for each $j\in[m]$ the following statements hold:
\begin{enumerate}
\item If $\frac{\partial J}{\partial \alpha_j}(\alpha^\star)-\frac{\nu^\star}{m} > 0$, then $\alpha_j^\star=0$.
\item If $\frac{\partial J}{\partial \alpha_j}(\alpha^\star)-\frac{\nu^\star}{m} < 0$, then $\alpha_j^\star=1$.
\item If $\frac{\partial J}{\partial \alpha_j}(\alpha^\star)-\frac{\nu^\star}{m} = 0$, then $\alpha_j^\star$ may be
interior or boundary; in particular, any interior coordinate $\alpha_j^\star\in(0,1)$ must satisfy the equality.
\end{enumerate}
\end{theorem}

\begin{proof}
Since $\tau<1$, there exists a strictly feasible point, for example $\alpha^0_j = (\tau+1)/2$ for all $j$, which satisfies
$0<\alpha^0_j<1$ and $\frac{1}{m}\sum_j \alpha^0_j = (\tau+1)/2 > \tau$. Hence Slater's condition holds, and KKT
conditions are necessary and sufficient for optimality.

The KKT conditions assert the existence of multipliers $(\lambda^\star,\mu^\star,\nu^\star)$ satisfying:
\begin{enumerate}
\item (Primal feasibility) $0\le \alpha_j^\star \le 1$ for all $j$ and $\frac{1}{m}\sum_{j=1}^m \alpha_j^\star \ge \tau$.
\item (Dual feasibility) $\lambda_j^\star\ge 0$, $\mu_j^\star\ge 0$ for all $j$, and $\nu^\star\ge 0$.
\item (Complementary slackness)
\begin{align*}
\lambda_j^\star \alpha_j^\star = 0,\qquad
\mu_j^\star(\alpha_j^\star-1)=0,\qquad
\nu^\star\left(\tau-\frac{1}{m}\sum_{j=1}^m \alpha_j^\star\right)=0.
\end{align*}
\item (Stationarity) For each $j\in[m]$,
\begin{align*}
\frac{\partial J}{\partial \alpha_j}(\alpha^\star) - \frac{\nu^\star}{m} - \lambda_j^\star + \mu_j^\star = 0.
\end{align*}
\end{enumerate}

Fix an index $j\in[m]$ and define the shifted gradient
\begin{align*}
g_j^\star := \frac{\partial J}{\partial \alpha_j}(\alpha^\star) - \frac{\nu^\star}{m}.
\end{align*}
By stationarity,
\begin{align*}
g_j^\star = \lambda_j^\star - \mu_j^\star.
\end{align*}
We now analyze cases using complementary slackness:

\begin{enumerate}
\item If $\alpha_j^\star=0$, then $\mu_j^\star(\alpha_j^\star-1)=\mu_j^\star(-1)=0$ implies $\mu_j^\star=0$, hence
$g_j^\star=\lambda_j^\star\ge 0$.

\item If $\alpha_j^\star=1$, then $\lambda_j^\star \alpha_j^\star=\lambda_j^\star=0$, hence
$g_j^\star = -\mu_j^\star \le 0$.

\item If $\alpha_j^\star\in(0,1)$, then $\lambda_j^\star \alpha_j^\star=0$ and $\alpha_j^\star>0$ imply $\lambda_j^\star=0$,
and $\mu_j^\star(\alpha_j^\star-1)=0$ and $\alpha_j^\star-1\neq 0$ imply $\mu_j^\star=0$. Therefore $g_j^\star=0$.
\end{enumerate}

These implications yield the desired conditional boundary rule:

\begin{enumerate}
\item If $g_j^\star>0$, then $\alpha_j^\star$ cannot equal $1$ (which would force $g_j^\star\le 0$) and cannot be interior
(which would force $g_j^\star=0$). Hence $\alpha_j^\star=0$.

\item If $g_j^\star<0$, then $\alpha_j^\star$ cannot equal $0$ (which would force $g_j^\star\ge 0$) and cannot be interior
(which would force $g_j^\star=0$). Hence $\alpha_j^\star=1$.

\item If $g_j^\star=0$, then the KKT conditions allow $\alpha_j^\star$ to be interior (in which case necessarily
$\lambda_j^\star=\mu_j^\star=0$) or boundary (for example $\alpha_j^\star=0$ with $\lambda_j^\star=0$ or
$\alpha_j^\star=1$ with $\mu_j^\star=0$). This proves the final claim.
\end{enumerate}
\end{proof}

\begin{corollary}[Thresholding structure for SKMM target selection]
\label{cor:skmm-thresholding}
Fix $w$ and consider the SKMM subproblem in $\alpha\in\mathbb{R}^m$ with yield constraint
\begin{align*}
\min_{\alpha}\quad & J(\alpha)
:=\left\|
\frac{1}{n}\sum_{i=1}^n w_i \varphi(x_i)
- \frac{1}{m}\sum_{j=1}^m \alpha_j \varphi(z_j)
\right\|_{\mathcal{H}}^2 \\
\text{s.t.}\quad & 0 \le \alpha_j \le 1 \quad \text{for all } j\in[m], \\
& \frac{1}{m}\sum_{j=1}^m \alpha_j \ge \tau,
\end{align*}
where $\tau<1$. Let $\alpha^\star$ be an optimal solution, and let $\nu^\star\ge 0$ be a KKT multiplier for the yield
constraint. Define the shifted partial derivatives
\begin{align*}
g_j^\star
:=
\frac{\partial J}{\partial \alpha_j}(\alpha^\star)-\frac{\nu^\star}{m},
\qquad j\in[m].
\end{align*}
Then:
\begin{enumerate}
\item If $g_j^\star>0$, then $\alpha_j^\star=0$.
\item If $g_j^\star<0$, then $\alpha_j^\star=1$.
\item If $\alpha_j^\star\in(0,1)$, then $g_j^\star=0$.
\end{enumerate}
Equivalently, there exists a scalar threshold $t^\star:=\nu^\star/m$ such that for every $j$,
\begin{align*}
\alpha_j^\star=0 \ \Rightarrow\ \frac{\partial J}{\partial \alpha_j}(\alpha^\star)\ge t^\star,
\qquad
\alpha_j^\star=1 \ \Rightarrow\ \frac{\partial J}{\partial \alpha_j}(\alpha^\star)\le t^\star,
\qquad
0<\alpha_j^\star<1 \ \Rightarrow\ \frac{\partial J}{\partial \alpha_j}(\alpha^\star)= t^\star.
\end{align*}
\end{corollary}

\begin{proof}
First note that for the SKMM objective, $J(\alpha)$ is a convex quadratic function of $\alpha$ because
\begin{align*}
J(\alpha)
&=
\left\|
\mu_w-\mu_\alpha
\right\|_{\mathcal{H}}^2,
\qquad
\mu_w:=\frac{1}{n}\sum_{i=1}^n w_i \varphi(x_i),
\qquad
\mu_\alpha:=\frac{1}{m}\sum_{j=1}^m \alpha_j \varphi(z_j),
\end{align*}
and expanding yields
\begin{align*}
J(\alpha)
=
\|\mu_w\|_{\mathcal{H}}^2
-
\frac{2}{m}\sum_{j=1}^m \alpha_j \langle \mu_w,\varphi(z_j)\rangle_{\mathcal{H}}
+
\frac{1}{m^2}\sum_{j=1}^m\sum_{l=1}^m \alpha_j\alpha_l\,k(z_j,z_l),
\end{align*}
whose Hessian with respect to $\alpha$ is $(2/m^2)\,K_{TT}$, where $K_{TT}=(k(z_j,z_l))_{j,l}$ is positive
semidefinite. Hence $J$ is convex and differentiable.

Since $\tau<1$, the point $\alpha^0_j=(\tau+1)/2$ is strictly feasible, so Slater's condition holds and KKT conditions
are necessary and sufficient for optimality. Therefore Theorem~\ref{lem:skmm-boundary-rule} applies directly to this
choice of $J$, yielding the three implications stated in the corollary. The equivalent threshold formulation follows
by defining $t^\star:=\nu^\star/m$.
\end{proof}

\section{Instability of Classifier-Based Density Ratios}
\label{app:classifier-instability}

A common alternative to KMM for estimating the density ratio $w(x) = P_T(x)/P_S(x)$ is probabilistic classification. By training a classifier $\hat{p}(x)$ to predict the probability that a sample $x$ belongs to the target domain, the estimated weight is given by $\hat{w}(x) = \hat{p}(x) / (1 - \hat{p}(x))$. The following proposition demonstrates why this approach can fail under covariate shift with limited support.

\begin{proposition}[Unbounded MMD under Bounded Classifier Error]
\label{prop:classifier-instability}
Let the true density ratio be $w^*(x) = P_T(x) / P_S(x)$ and let $\hat{w}(x)$ be the density ratio estimated via a classifier $\hat{p}(x)$. For any $\epsilon > 0$ and $M > 0$, there exists a sufficiently small $\delta$ and a sample containing such points such that $\mathrm{MMD}(P_S^{\hat w}, P_T) > M$.
\end{proposition}

\begin{proof}
Assuming equal prior class probabilities for the source and target domains, the true conditional probability of a sample belonging to the target domain is $p^*(x) = P_T(x) / (P_T(x) + P_S(x))$. The true density ratio is exactly $w^*(x) = p^*(x) / (1 - p^*(x))$.

Consider a region in the covariate space where the source distribution has very low mass but the target distribution is well-supported. In this region, $P_S(x) \to 0$, which implies $p^*(x) \to 1$. Let $p^*(x) = 1 - \delta$ for some sufficiently small $\delta > 0$. The true weight is:
$$w^*(x) = \frac{1 - \delta}{\delta} \approx \frac{1}{\delta}.$$

Assume our classifier slightly overestimates the probability of the target class such that $\hat{p}(x) = 1 - \delta^2$. The absolute error of our classifier is:
$$|\hat{p}(x) - p^*(x)| = |(1 - \delta^2) - (1 - \delta)| = \delta - \delta^2 < \delta.$$
By choosing $\delta \le \epsilon$, the classifier satisfies the bounded error condition. However, the estimated weight in this region becomes:
$$\hat{w}(x) = \frac{1 - \delta^2}{\delta^2} \approx \frac{1}{\delta^2}.$$
The deviation of the estimated weight from the true weight is:
$$|\hat{w}(x) - w^*(x)| = \left| \frac{1 - \delta^2}{\delta^2} - \frac{1 - \delta}{\delta} \right| = \frac{1 - \delta}{\delta^2}.$$
As $\delta \to 0$, the difference between the estimated weight and the true weight grows as $\mathcal{O}(1/\delta^2)$.

Recall the empirical MMD objective in the RKHS $\mathcal{H}$ with feature map $\varphi$:
$$\mathrm{MMD}(P_S^{\hat{w}}, P_T) = \left\| \frac{1}{n}\sum_{i=1}^n \hat{w}(x_i) \varphi(x_i) - \frac{1}{m}\sum_{j=1}^m \varphi(z_j) \right\|_{\mathcal{H}}.$$
Because $\hat{w}(x_i)$ grows unboundedly as $\delta \to 0$ for points in this low-density region, the term $\hat{w}(x_i) \varphi(x_i)$ will dominate the empirical source mean. Thus, for any arbitrary bound $M$, there exists a sufficiently small $\delta$ and a sample containing such points such that $\mathrm{MMD}(P_S^{\hat{w}}, P_T) > M$. 

This demonstrates that minimizing a classifier's probability estimation error does not safely bound the MMD or the variance of the weights, leading to degraded conformal prediction coverage. KMM explicitly resolves this by enforcing the bound $w_i \le B$ and directly matching the moments. \qed
\end{proof}

\pagebreak

\section{Mondrian Conformal Prediction}
\label{app:mondrian}

\subsection{Mondrian Conformal Prediction}

Mondrian Conformal Prediction (MCP) extends standard conformal prediction by conditioning calibration on a predefined partition of the data \cite{vovk2012conditional}. Instead of enforcing marginal coverage over the entire population, MCP guarantees coverage within each subgroup defined by a class label or other stratification variable.

Let $g(X) \in \mathcal{G}$ denote a grouping function. In classification tasks, a common choice is $g(X,Y)=Y$, leading to class-conditional calibration. For each group $c \in \mathcal{G}$, define the group-specific conformity scores
\[
\{ S_i : g(X_i,Y_i)=c \}.
\]
Let $n_c$ denote the number of calibration points in group $c$. The group-specific empirical CDF is
\[
F_c(t) = \frac{1}{n_c} \sum_{i: g(X_i,Y_i)=c} \mathbf{1}\{S_i \le t\}.
\]
The Mondrian conformal threshold for group $c$ is then
\[
\hat{q}_{1-\alpha}^{(c)} 
= \inf \left\{ t : F_c(t) \ge 1-\alpha \right\}.
\]

The resulting prediction set for a new input $(X_{N+1},Y_{N+1})$ with group $c = g(X_{N+1})$ satisfies
\[
\mathbb{P}\bigl( Y_{N+1} \in \hat{C}(X_{N+1}) \mid g(X_{N+1}) = c \bigr)
\ge 1-\alpha,
\]
under exchangeability within each group.

In imbalanced molecular datasets, class-conditional calibration prevents dominant classes from overwhelming rare but important categories (e.g., active or toxic compounds).

\subsection{Weighted Mondrian Conformal Prediction}

Under covariate shift, we extend MCP by incorporating importance weights within each group. Let $w_i$ denote the calibration weights (e.g., obtained via KMM or SKMM). Define the normalized group-specific weights
\[
\tilde{w}_i^{(c)} 
= \frac{w_i}{\sum_{j: g(X_j,Y_j)=c} w_j},
\quad \text{for } g(X_i,Y_i)=c.
\]
The weighted group CDF becomes
\[
F_c^{w}(t) 
= \sum_{i: g(X_i,Y_i)=c} 
\tilde{w}_i^{(c)} \mathbf{1}\{S_i \le t\}.
\]
The Mondrian weighted quantile is
\[
\hat{q}_{1-\alpha}^{(c),w}
= \inf \left\{ t : F_c^{w}(t) \ge 1-\alpha \right\}.
\]

The prediction set is constructed exactly as in the global weighted case, except that quantiles are computed separately within each group.

\begin{algorithm}[t]
\caption{Mondrian Split Conformal Prediction with KMM / Selective KMM}
\label{alg:kmm_mondrian}
\begin{algorithmic}[1]
\REQUIRE Labeled dataset $\mathcal{D} = \{(x_i,y_i)\}_{i=1}^N$, 
unlabeled test covariates $\{z_j\}_{j=1}^m$, 
miscoverage level $\alpha_{\mathrm{conf}}$, 
kernel $k(\cdot,\cdot)$, 
weight bound $B$, 
selection fraction $\tau$ (optional)

\ENSURE Prediction sets $\hat C(z_j)$

\STATE Split $\mathcal{D}$ into training set $\mathcal{D}_{\mathrm{train}}$ 
and calibration set $\mathcal{D}_{\mathrm{cal}} = \{(x_i,y_i)\}_{i=1}^n$

\STATE Train predictor $f$ on $\mathcal{D}_{\mathrm{train}}$

\STATE Define conformity score $S(x,y)$ based on $f$

\IF{Selective KMM is enabled}
    \STATE Solve double-weighted KMM to obtain selection weights $\alpha_j$
    \STATE Define selected subset $\mathcal{T}_{\mathrm{sel}}$
\ELSE
    \STATE Set $\mathcal{T}_{\mathrm{sel}} = \{z_j\}_{j=1}^m$
\ENDIF

\STATE Solve standard KMM between $\mathcal{D}_{\mathrm{cal}}$ and $\mathcal{T}_{\mathrm{sel}}$ to obtain source weights $w_i$

\FOR{each class $c$}
    \STATE Let $\mathcal{I}_c = \{ i : y_i = c \}$
    \STATE Normalize weights within class:
    \[
    \tilde w_i^{(c)} = 
    \frac{w_i}{\sum_{k \in \mathcal{I}_c} w_k}
    \quad \text{for } i \in \mathcal{I}_c
    \]
    \STATE Compute class-conditional quantile $\hat q_{1-\alpha}^{(c)}$:
    \[
    \sum_{i \in \mathcal{I}_c} 
    \tilde w_i^{(c)} 
    \mathbf{1}\{S_i \le t\} 
    \ge 1-\alpha_{\mathrm{conf}}
    \]
\ENDFOR

\FOR{each $z_j \in \mathcal{T}_{\mathrm{sel}}$}
    \STATE Let $\hat c = \text{predicted class of } z_j$
    \STATE Output
    \[
    \hat C(z_j) = 
    \{ y : S(z_j,y) \le \hat q_{1-\alpha}^{(\hat c)} \}
    \]
\ENDFOR
\end{algorithmic}
\end{algorithm}

\paragraph{Practical Choice of Groups.}
In our experiments, we use class-conditional Mondrian calibration: $g(X,Y)=Y$. For toxicity datasets, we calibrate within the non-toxic class to avoid undercoverage of rare toxic compounds; for activity datasets, we calibrate within the active class. This choice reflects domain-specific asymmetry in decision importance. In Algorithm~\ref{alg:kmm_mondrian}, we summarize the application of KMM-CP for class conditional guarentees under covariate shift. 

\clearpage
\section{Additional Experimental Details}
\label{app:exp}

\subsection{Code Availability}

All code used to reproduce the experiments is publicly available at:

\begin{center}
\url{https://github.com/siddharthal/KMM-CP}
\end{center}

The repository contains complete training scripts, data preprocessing pipelines, and exact random seeds for all five runs per experiment. All results reported in the paper are fully reproducible.

\subsection{Dataset Statistics}

Table~\ref{tab:dataset_stats} summarizes dataset sizes and splits.

\begin{table}[h]
\centering
\caption{Dataset statistics and data splits.}
\label{tab:dataset_stats}
\small
\begin{tabular}{lcccc}
\toprule
Dataset & Total $N$ & Train & Calib & Test \\
\midrule
AMES         & 7,280  & 5,096 & 1,092 & 1,092 \\
BBB-Martins  & 2,033  & 1,423 & 305   & 305   \\
HIV          & 41,132 & 28,793 & 6,169 & 6,170 \\
Tox21        & 7,267  & 5,087 & 1,090 & 1,090 \\
hERG         & 13,445 & 9,411 & 2,017 & 2,017 \\
\bottomrule
\end{tabular}
\end{table}

\subsection{Training Details and Hyperparameters}

All models use the AttentiveFP \cite{attentivefp} graph neural network architecture implemented via the \texttt{dgllife} library \cite{dgllife}, with training performed in PyTorch.

\paragraph{Optimization.}
\begin{itemize}
    \item Optimizer: Adam
    \item Learning rate: $1 \times 10^{-3}$
    \item Learning rate scheduler: StepLR (step\_size=1, $\gamma=0.8$)
\end{itemize}

\paragraph{Training Loop.}
\begin{itemize}
    \item Maximum epochs: 50
    \item Batch size: 128
    \item Early stopping: patience = 5 (based on training loss)
\end{itemize}

\paragraph{Model Architecture.}
\begin{itemize}
    \item Graph embedding dimension: 512
    \item Hidden dimension (classifier): 64
    \item GNN layers (AttentiveFP steps): 3
    \item Output dimension: 2 (binary classification)
    \item Total parameters: 10M
\end{itemize}

\paragraph{Training time} Model training takes approximately 5–25 minutes per dataset on a single NVIDIA A100 GPU, depending on dataset size.

\paragraph{Loss Function.}
\begin{itemize}
    \item Cross-entropy loss
    \item Class weights enabled (inverse class frequency)
\end{itemize}

\subsection{Conformal Prediction Parameters}

\paragraph{Non-conformity Score.}
Non-conformity scores are defined using the predicted class probabilities from the trained AttentiveFP model.

\paragraph{Feature Representation.}
Penultimate-layer embeddings from the trained model are used as feature representations for density-ratio estimation and KMM.

\subsection{Implementation of KMM}

Kernel Mean Matching (KMM) optimization is solved using \texttt{cvxopt} library.

\paragraph{Weight Constraints.}
\begin{itemize}
    \item Weight upper bound: $B = 30$
    \item Target selection fraction (SKMM): $\tau = 0.5$
    \item Second-stage filtering threshold on $\alpha$: 0.2
\end{itemize}

These hyperparameters are fixed across all datasets and are not tuned per dataset.

\paragraph{Kernel Bandwidth Selection.}

We use an RBF kernel with bandwidth selected via a median-distance heuristic \cite{gretton2012kernel}. Specifically, we compute the median pairwise distance between calibration and test embeddings and evaluate bandwidths from the set:
\[
\sigma \in \{0.01,\, 0.1,\, 0.5,\, 1,\, 2\} \times \mathrm{median\_distance}.
\]
The bandwidth maximizing standardized MMD (z-scored across permutations) is selected.

This strategy balances sensitivity to local discrepancies and stability under high-dimensional embeddings, and is commonly used in kernel two-sample testing.

\paragraph{Computation Time.}

KMM optimization is performed on CPUs and requires approximately:
\begin{itemize}
    \item 6 minutes for the smallest dataset,
    \item up to 50 minutes for the largest dataset.
\end{itemize}

\subsection{Compute Resources}

All experiments were conducted on a cluster with:
\begin{itemize}
    \item 4 NVIDIA A100 GPUs,
    \item 64 CPU cores.
\end{itemize}

\subsection{Implementation of Density-Ratio Baselines}

\paragraph{KDE Baseline.}
For the KDE-based density-ratio baseline, we estimate source and target marginal densities using a Gaussian kernel. The bandwidth is selected via a held-out validation split by maximizing the log-likelihood over candidate values $\{0.01, 0.1, 1.0, 10\}$. The same bandwidth is used for both source and target density estimation. The density ratio is then formed as $\hat{w}(x) = \hat{p}_t(x)/\hat{p}_s(x)$.

\paragraph{Logistic Regression Baseline.}
For the classifier-based density-ratio baseline, we train a logistic regression classifier to distinguish calibration and test samples using 64-dimensional embeddings extracted from the final layer of the trained predictor. The classifier is implemented using the \texttt{scikit-learn} library with default regularization settings. Density ratios are obtained from the predicted probabilities via the standard odds transformation.

\section{Additional results}
\label{app:results}
\subsection{Calibration coverage plots for global CP}
Figure~\ref{fig:appendix-global-coverage} presents calibration curves under covariate shift for global conformal prediction across all datasets. Each curve plots realized coverage against target coverage levels from $0.5$ to $0.95$, averaged over five random runs. The dashed line indicates perfect calibration.

Under covariate shift, uniform conformal prediction systematically undercovers, with deviations increasing at higher target coverage levels. Density-ratio baselines partially correct this behavior but exhibit residual miscalibration and variability across datasets. In particular, classifier- and KDE-based weighting methods often fail to fully align with the nominal coverage line at moderate-to-high coverage levels.

In contrast, KMM produces curves that more closely track the nominal diagonal, indicating improved distributional alignment. The proposed selective KMM (SKMM) further tightens this alignment, yielding the smallest deviations from nominal coverage across datasets. The improvement is consistent across both toxicity and activity prediction tasks, and persists across the full range of coverage levels.

These plots visually corroborate the mondran CP trends and ranking results reported in the main text.

\begin{figure*}[t]
  \centering
  \includegraphics[width=\textwidth]{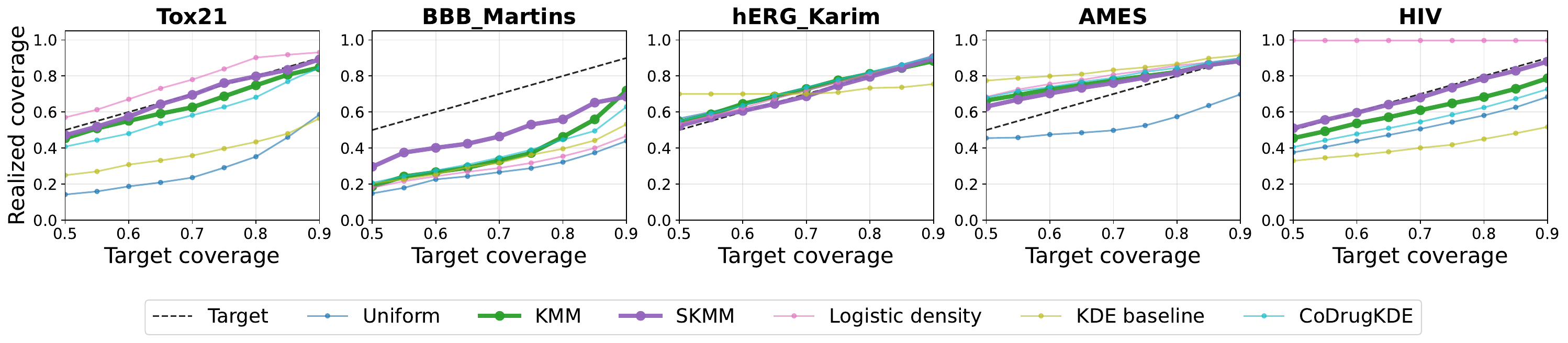}
  \caption{Calibration coverage under global conformal prediction across datasets.}
  \label{fig:appendix-global-coverage}
\end{figure*}


\subsection{Additional Bias-Variance and Selective Filtering plots}

In this section, we provide supplemental plots to support the analysis detailed in Section~\ref{sec:bias-variance} for the remaining datasets: AMES, hERG, and HIV.Figure~\ref{fig:bias-variance-supp} presents the bias-variance decomposition across these benchmarks, while Figure~\ref{fig:alpha-dist-supp} illustrates the corresponding $\alpha$ distribution plots. Overall, the bias-variance trends observed here closely mimic the results shown in Figure~\ref{fig:full-shift-mmd-ess}; specifically, the KMM-based methods continue to exhibit higher stability across varying shift intensities. Similarly, the $\alpha$ distributions demonstrate emergent sparsity across nearly all experimental settings. However, the hERG\_Karim dataset exhibits a denser $\alpha$ distribution; this is potentially attributable to a higher intrinsic overlap in the covariate support.

\begin{figure}[htbp]
    \centering
    \begin{subfigure}[b]{0.52\textwidth}
        \centering
        \includegraphics[width=\textwidth]{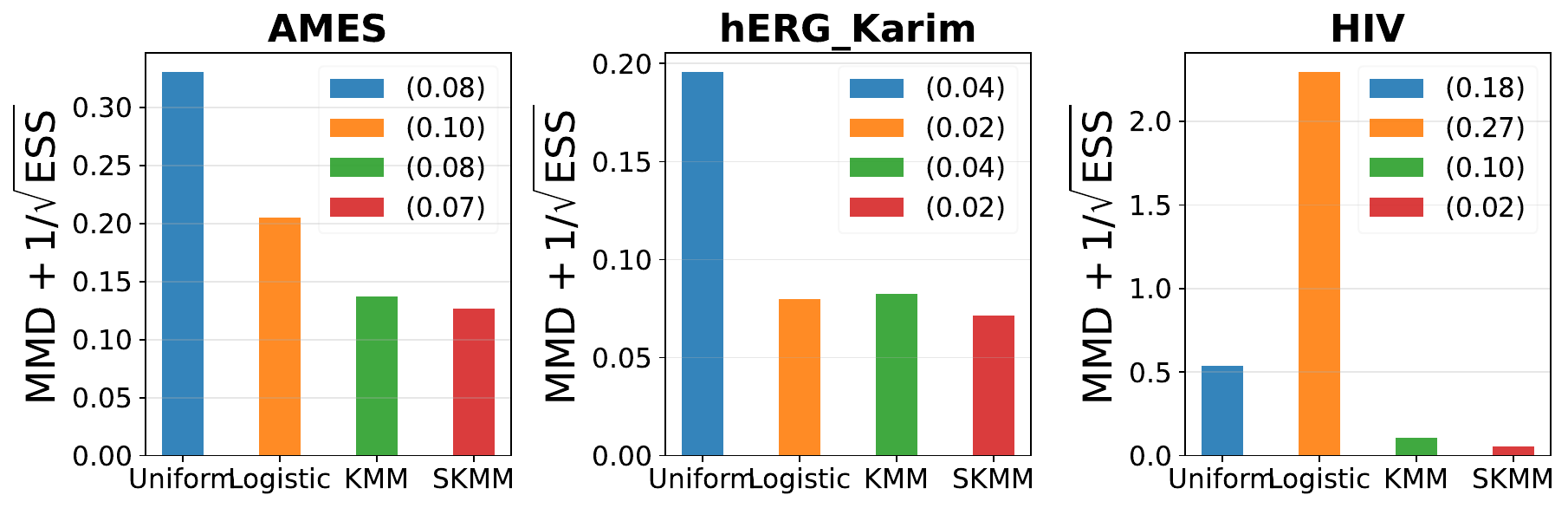}
        \caption{Bias-Variance (AMES)}
        \label{fig:bias-variance-supp}
    \end{subfigure}
    \hfill
    \\
    \begin{subfigure}[b]{0.52\textwidth}
        \centering
        \includegraphics[width=\textwidth]{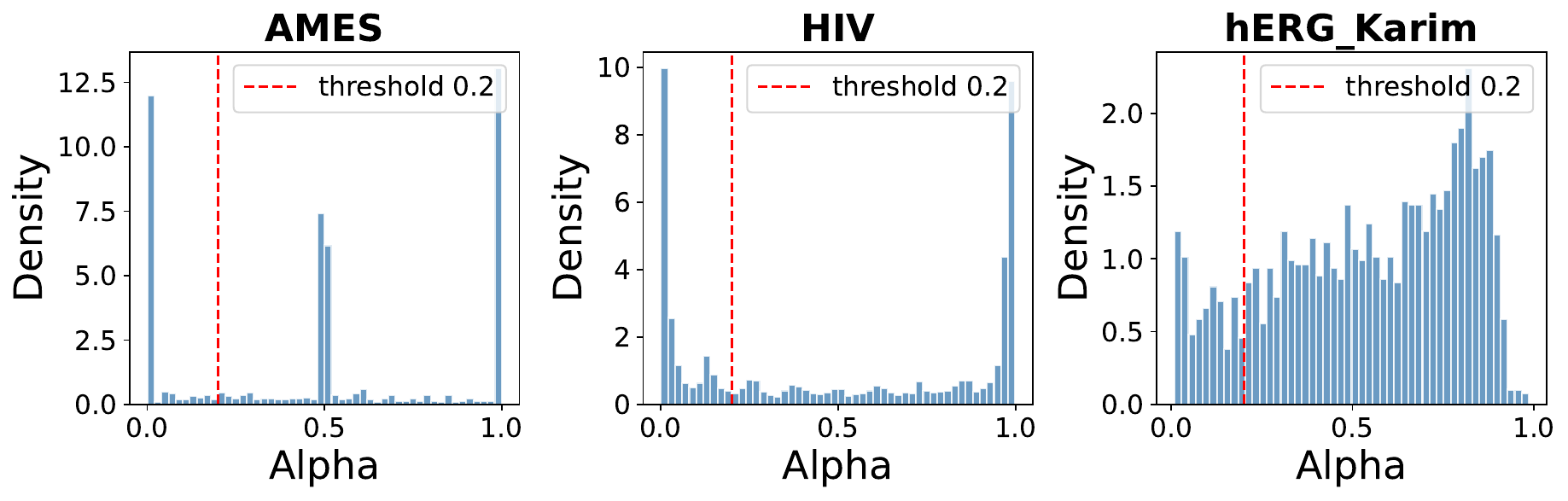}
        \caption{$\alpha$ Distribution (AMES)}
        \label{fig:alpha-dist-supp}
    \end{subfigure}
    
    \caption{Figure. (a) illustrates the bias-variance proxy($\mathrm{MMD} + \sqrt{1/\mathrm{ESS}}$) , while (b) displays the emergent sparsity in the weight distribution.}
\end{figure}

\subsection{Filtered Fraction}
\label{app:filtered_fraction}

Table~\ref{tab:app_filtered_fraction} reports the average filtration ratio under global and Mondrian calibration, defined as the percentage of test instances retained after the selective KMM filtering stage (mean $\pm$ standard deviation across five runs). Recall that the target selection fraction is set to $\tau = 0.5$, corresponding to retaining approximately 50\% of the test instances.

Across datasets, the observed retention rates range from roughly 58\% to 77\%, indicating that the learned selection variables $\alpha_j$ concentrate mass on a substantial but restricted subset of the test distribution. Also note that SKMM does not collapse to extreme filtering: a majority of test instances are retained across all benchmarks. This confirms that the method focuses correction on regions where density-ratio estimation is stable, rather than aggressively discarding data.

\begin{table}[t]
\centering
\caption{Average filtration ratio (\%) under two-stage Selective KMM (mean $\pm$ std).}
\label{tab:filtration_ratio}
\small
\begin{tabular}{lcc}
\toprule
Dataset &  Fraction Filtered \\
\midrule
AMES         & 71.2 $\pm$ 5.9 \\
Tox21        & 58.3 $\pm$ 0.5 \\
BBB-Martins  & 60.1 $\pm$ 10.1 \\
hERG-Karim   & 76.7 $\pm$ 5.9 \\
HIV          & 62.2 $\pm$ 1.4 \\
\bottomrule
\end{tabular}
\label{tab:app_filtered_fraction}
\end{table}

\subsection{Full Calibration Results Across Datasets}
\label{app:full_results}

The following tables report full calibration curves (mean $\pm$ std over five runs) for all datasets under both global and Mondrian conformal prediction across target coverages from $0.5$ to $0.95$. Best results per column are shown in \textbf{bold}, and second best in \textcolor{gray}{\textbf{gray bold}}.

\paragraph{Overall trends.}
Across all datasets and coverage levels, SKMM consistently achieves the lowest or near-lowest deviation from nominal coverage. This pattern holds for both global and class-conditional (Mondrian) calibration. KMM (Beta-only) typically ranks second, while classifier- and KDE-based baselines exhibit larger systematic deviations, particularly at higher target coverage levels.

\paragraph{Global calibration.}
Under global conformal prediction, uniform weighting frequently exhibits pronounced undercoverage as the target coverage increases. Density-ratio baselines partially mitigate this effect but remain unstable across datasets. In contrast, KMM substantially improves calibration alignment, and SKMM further reduces deviation across nearly all coverage levels. The average rank columns confirm this consistency: SKMM achieves the best mean rank on every dataset under global calibration.

\paragraph{Mondrian calibration} results show analogous results for Mondrian conformal prediction. The same qualitative behavior persists: SKMM maintains the strongest alignment with nominal coverage across coverage levels, followed by KMM. Importantly, improvements are not restricted to a specific coverage region; gains remain stable from moderate ($0.6$--$0.7$) to high ($0.9$--$0.95$) target coverage.

\small

\begin{table}[t]
\centering
\caption{AMES --- Global calibration (mean $\pm$ std). Best is \textbf{bold}; second best is \textcolor{gray}{\textbf{gray bold}}.}
\resizebox{\textwidth}{!}{\begin{tabular}{lccccccccccc}
\toprule
Method & 0.5 & 0.55 & 0.6 & 0.65 & 0.7 & 0.75 & 0.8 & 0.85 & 0.9 & 0.95 & Avg. Rank \\
\midrule
Uniform & \textbf{0.455$\pm$0.000} & \textbf{0.458$\pm$0.004} & \textcolor{gray}{\textbf{0.475$\pm$0.006}} & 0.485$\pm$0.004 & 0.498$\pm$0.005 & 0.525$\pm$0.014 & 0.574$\pm$0.034 & 0.636$\pm$0.040 & 0.698$\pm$0.034 & 0.824$\pm$0.048 & 4.60 \\
KDE & 0.774$\pm$0.127 & 0.788$\pm$0.113 & 0.799$\pm$0.104 & 0.810$\pm$0.092 & 0.833$\pm$0.081 & 0.848$\pm$0.073 & 0.866$\pm$0.051 & 0.898$\pm$0.046 & 0.915$\pm$0.036 & \textbf{0.933$\pm$0.022} & 4.80 \\
CoDrugKDE & 0.680$\pm$0.041 & 0.716$\pm$0.040 & 0.738$\pm$0.042 & 0.767$\pm$0.035 & 0.790$\pm$0.034 & 0.823$\pm$0.042 & 0.845$\pm$0.033 & 0.874$\pm$0.027 & \textbf{0.899$\pm$0.033} & \textcolor{gray}{\textbf{0.933$\pm$0.027}} & 3.00 \\
Logistic & 0.684$\pm$0.058 & 0.725$\pm$0.068 & 0.755$\pm$0.051 & 0.778$\pm$0.037 & 0.807$\pm$0.030 & 0.831$\pm$0.036 & 0.858$\pm$0.044 & 0.875$\pm$0.028 & \textcolor{gray}{\textbf{0.898$\pm$0.027}} & 0.925$\pm$0.030 & 4.10 \\
KMM & 0.664$\pm$0.026 & 0.694$\pm$0.025 & 0.726$\pm$0.021 & \textcolor{gray}{\textbf{0.750$\pm$0.012}} & \textcolor{gray}{\textbf{0.775$\pm$0.017}} & \textcolor{gray}{\textbf{0.799$\pm$0.021}} & \textcolor{gray}{\textbf{0.822$\pm$0.027}} & \textbf{0.861$\pm$0.017} & 0.885$\pm$0.024 & 0.915$\pm$0.019 & 2.60 \\
SKMM & \textcolor{gray}{\textbf{0.630$\pm$0.028}} & \textcolor{gray}{\textbf{0.669$\pm$0.024}} & \textbf{0.703$\pm$0.019} & \textbf{0.734$\pm$0.009} & \textbf{0.761$\pm$0.010} & \textbf{0.791$\pm$0.020} & \textbf{0.819$\pm$0.029} & \textcolor{gray}{\textbf{0.862$\pm$0.022}} & 0.883$\pm$0.025 & 0.927$\pm$0.020 & 1.90 \\
\bottomrule
\end{tabular}}
\end{table}

\begin{table}[t]
\centering
\caption{AMES --- Mondrian calibration (mean $\pm$ std). Best is \textbf{bold}; second best is \textcolor{gray}{\textbf{gray bold}}.}
\resizebox{\textwidth}{!}{\begin{tabular}{lccccccccccc}
\toprule
Method & 0.5 & 0.55 & 0.6 & 0.65 & 0.7 & 0.75 & 0.8 & 0.85 & 0.9 & 0.95 & Avg. Rank \\
\midrule
Uniform & \textbf{0.470$\pm$0.006} & \textbf{0.499$\pm$0.014} & \textcolor{gray}{\textbf{0.530$\pm$0.025}} & \textcolor{gray}{\textbf{0.569$\pm$0.035}} & 0.606$\pm$0.045 & 0.642$\pm$0.033 & 0.683$\pm$0.035 & 0.732$\pm$0.038 & 0.809$\pm$0.048 & 0.916$\pm$0.055 & 4.00 \\
KDE & \textcolor{gray}{\textbf{0.585$\pm$0.080}} & \textcolor{gray}{\textbf{0.605$\pm$0.083}} & \textbf{0.613$\pm$0.094} & \textbf{0.642$\pm$0.096} & \textbf{0.658$\pm$0.112} & \textbf{0.704$\pm$0.090} & 0.758$\pm$0.085 & 0.808$\pm$0.053 & 0.817$\pm$0.064 & 0.847$\pm$0.086 & 2.80 \\
CoDrugKDE & 0.685$\pm$0.031 & 0.710$\pm$0.034 & 0.735$\pm$0.039 & 0.764$\pm$0.038 & 0.792$\pm$0.032 & 0.820$\pm$0.028 & 0.838$\pm$0.033 & 0.865$\pm$0.031 & \textbf{0.901$\pm$0.026} & 0.926$\pm$0.030 & 4.10 \\
Logistic & 0.691$\pm$0.052 & 0.722$\pm$0.033 & 0.746$\pm$0.033 & 0.777$\pm$0.038 & 0.795$\pm$0.041 & 0.819$\pm$0.043 & 0.847$\pm$0.029 & 0.865$\pm$0.031 & \textcolor{gray}{\textbf{0.894$\pm$0.020}} & 0.926$\pm$0.030 & 4.70 \\
KMM & 0.669$\pm$0.020 & 0.698$\pm$0.024 & 0.723$\pm$0.029 & 0.744$\pm$0.030 & 0.767$\pm$0.030 & 0.803$\pm$0.026 & \textbf{0.834$\pm$0.029} & \textbf{0.855$\pm$0.027} & 0.881$\pm$0.024 & \textcolor{gray}{\textbf{0.926$\pm$0.028}} & 3.00 \\
SKMM & 0.655$\pm$0.013 & 0.680$\pm$0.020 & 0.706$\pm$0.022 & 0.740$\pm$0.026 & \textcolor{gray}{\textbf{0.765$\pm$0.032}} & \textcolor{gray}{\textbf{0.801$\pm$0.029}} & \textcolor{gray}{\textbf{0.837$\pm$0.032}} & \textcolor{gray}{\textbf{0.857$\pm$0.028}} & 0.889$\pm$0.028 & \textbf{0.936$\pm$0.026} & 2.40 \\
\bottomrule
\end{tabular}}
\end{table}

\begin{table}[t]
\centering
\caption{BBB\_Martins --- Global calibration (mean $\pm$ std). Best is \textbf{bold}; second best is \textcolor{gray}{\textbf{gray bold}}.}
\resizebox{\textwidth}{!}{\begin{tabular}{lccccccccccc}
\toprule
Method & 0.5 & 0.55 & 0.6 & 0.65 & 0.7 & 0.75 & 0.8 & 0.85 & 0.9 & 0.95 & Avg. Rank \\
\midrule
Uniform & 0.148$\pm$0.053 & 0.179$\pm$0.072 & 0.226$\pm$0.069 & 0.243$\pm$0.078 & 0.266$\pm$0.081 & 0.288$\pm$0.083 & 0.322$\pm$0.077 & 0.374$\pm$0.064 & 0.439$\pm$0.087 & 0.659$\pm$0.110 & 6.00 \\
KDE & 0.205$\pm$0.094 & 0.226$\pm$0.112 & 0.254$\pm$0.116 & 0.294$\pm$0.112 & 0.317$\pm$0.114 & 0.362$\pm$0.125 & 0.396$\pm$0.142 & 0.442$\pm$0.173 & 0.532$\pm$0.195 & 0.698$\pm$0.230 & 3.80 \\
CoDrugKDE & \textcolor{gray}{\textbf{0.207$\pm$0.098}} & \textcolor{gray}{\textbf{0.244$\pm$0.100}} & \textcolor{gray}{\textbf{0.275$\pm$0.100}} & \textcolor{gray}{\textbf{0.308$\pm$0.102}} & \textcolor{gray}{\textbf{0.347$\pm$0.089}} & \textcolor{gray}{\textbf{0.390$\pm$0.084}} & 0.446$\pm$0.093 & 0.495$\pm$0.108 & 0.629$\pm$0.080 & 0.770$\pm$0.055 & 2.40 \\
Logistic & 0.180$\pm$0.068 & 0.216$\pm$0.066 & 0.243$\pm$0.077 & 0.268$\pm$0.078 & 0.289$\pm$0.079 & 0.317$\pm$0.073 & 0.354$\pm$0.067 & 0.400$\pm$0.059 & 0.467$\pm$0.072 & 0.693$\pm$0.084 & 5.00 \\
KMM & 0.187$\pm$0.061 & 0.242$\pm$0.108 & 0.266$\pm$0.119 & 0.291$\pm$0.111 & 0.331$\pm$0.096 & 0.372$\pm$0.084 & \textcolor{gray}{\textbf{0.462$\pm$0.125}} & \textcolor{gray}{\textbf{0.559$\pm$0.110}} & \textbf{0.719$\pm$0.112} & \textbf{0.838$\pm$0.038} & 2.60 \\
SKMM & \textbf{0.296$\pm$0.121} & \textbf{0.375$\pm$0.171} & \textbf{0.401$\pm$0.166} & \textbf{0.424$\pm$0.157} & \textbf{0.464$\pm$0.142} & \textbf{0.530$\pm$0.135} & \textbf{0.559$\pm$0.138} & \textbf{0.652$\pm$0.094} & \textcolor{gray}{\textbf{0.685$\pm$0.091}} & \textcolor{gray}{\textbf{0.812$\pm$0.086}} & 1.20 \\
\bottomrule
\end{tabular}}
\end{table}

\begin{table}[t]
\centering
\caption{BBB\_Martins --- Mondrian calibration (mean $\pm$ std). Best is \textbf{bold}; second best is \textcolor{gray}{\textbf{gray bold}}.}
\resizebox{\textwidth}{!}{\begin{tabular}{lccccccccccc}
\toprule
Method & 0.5 & 0.55 & 0.6 & 0.65 & 0.7 & 0.75 & 0.8 & 0.85 & 0.9 & 0.95 & Avg. Rank \\
\midrule
Uniform & 0.189$\pm$0.005 & 0.202$\pm$0.002 & 0.208$\pm$0.003 & 0.216$\pm$0.004 & 0.222$\pm$0.004 & 0.236$\pm$0.010 & 0.252$\pm$0.010 & 0.287$\pm$0.022 & 0.367$\pm$0.069 & 0.683$\pm$0.088 & 5.70 \\
KDE & 0.214$\pm$0.015 & 0.226$\pm$0.022 & 0.241$\pm$0.020 & \textcolor{gray}{\textbf{0.269$\pm$0.020}} & 0.282$\pm$0.025 & 0.308$\pm$0.040 & 0.343$\pm$0.072 & 0.408$\pm$0.086 & 0.473$\pm$0.121 & 0.632$\pm$0.173 & 3.90 \\
CoDrugKDE & 0.218$\pm$0.014 & \textcolor{gray}{\textbf{0.235$\pm$0.020}} & \textcolor{gray}{\textbf{0.248$\pm$0.026}} & 0.265$\pm$0.027 & \textcolor{gray}{\textbf{0.291$\pm$0.041}} & 0.315$\pm$0.044 & 0.357$\pm$0.067 & 0.432$\pm$0.037 & 0.511$\pm$0.062 & \textbf{0.687$\pm$0.075} & 2.50 \\
Logistic & 0.202$\pm$0.003 & 0.209$\pm$0.006 & 0.218$\pm$0.005 & 0.226$\pm$0.011 & 0.237$\pm$0.013 & 0.252$\pm$0.019 & 0.272$\pm$0.018 & 0.311$\pm$0.047 & 0.386$\pm$0.065 & 0.658$\pm$0.090 & 5.00 \\
KMM & \textcolor{gray}{\textbf{0.219$\pm$0.013}} & 0.227$\pm$0.018 & 0.242$\pm$0.015 & 0.264$\pm$0.023 & 0.279$\pm$0.013 & \textcolor{gray}{\textbf{0.353$\pm$0.078}} & \textcolor{gray}{\textbf{0.437$\pm$0.027}} & \textcolor{gray}{\textbf{0.496$\pm$0.070}} & \textbf{0.604$\pm$0.089} & \textcolor{gray}{\textbf{0.686$\pm$0.085}} & 2.50 \\
SKMM & \textbf{0.272$\pm$0.037} & \textbf{0.312$\pm$0.031} & \textbf{0.329$\pm$0.028} & \textbf{0.383$\pm$0.091} & \textbf{0.433$\pm$0.094} & \textbf{0.453$\pm$0.112} & \textbf{0.510$\pm$0.112} & \textbf{0.575$\pm$0.082} & \textcolor{gray}{\textbf{0.604$\pm$0.092}} & 0.672$\pm$0.095 & 1.40 \\
\bottomrule
\end{tabular}}
\end{table}

\begin{table}[t]
\centering
\caption{HIV --- Global calibration (mean $\pm$ std). Best is \textbf{bold}; second best is \textcolor{gray}{\textbf{gray bold}}.}
\resizebox{\textwidth}{!}{\begin{tabular}{lccccccccccc}
\toprule
Method & 0.5 & 0.55 & 0.6 & 0.65 & 0.7 & 0.75 & 0.8 & 0.85 & 0.9 & 0.95 & Avg. Rank \\
\midrule
Uniform & 0.376$\pm$0.063 & 0.406$\pm$0.063 & 0.439$\pm$0.064 & 0.472$\pm$0.064 & 0.506$\pm$0.064 & 0.544$\pm$0.063 & 0.581$\pm$0.065 & 0.626$\pm$0.065 & 0.684$\pm$0.061 & 0.779$\pm$0.049 & 4.40 \\
KDE & 0.329$\pm$0.124 & 0.346$\pm$0.126 & 0.361$\pm$0.132 & 0.379$\pm$0.137 & 0.401$\pm$0.144 & 0.418$\pm$0.160 & 0.450$\pm$0.168 & 0.482$\pm$0.188 & 0.517$\pm$0.203 & 0.577$\pm$0.255 & 5.60 \\
CoDrugKDE & 0.405$\pm$0.101 & 0.443$\pm$0.105 & 0.478$\pm$0.110 & 0.510$\pm$0.108 & 0.545$\pm$0.112 & 0.585$\pm$0.121 & 0.625$\pm$0.116 & 0.674$\pm$0.110 & 0.727$\pm$0.089 & 0.803$\pm$0.091 & 3.30 \\
Logistic & 0.997$\pm$0.002 & 0.997$\pm$0.002 & 0.997$\pm$0.002 & 0.997$\pm$0.002 & 0.997$\pm$0.002 & 0.997$\pm$0.002 & 0.997$\pm$0.002 & 0.997$\pm$0.002 & \textcolor{gray}{\textbf{0.997$\pm$0.002}} & \textcolor{gray}{\textbf{0.997$\pm$0.002}} & 4.50 \\
KMM & \textcolor{gray}{\textbf{0.455$\pm$0.058}} & \textcolor{gray}{\textbf{0.493$\pm$0.056}} & \textcolor{gray}{\textbf{0.537$\pm$0.048}} & \textcolor{gray}{\textbf{0.571$\pm$0.042}} & \textcolor{gray}{\textbf{0.610$\pm$0.042}} & \textcolor{gray}{\textbf{0.648$\pm$0.033}} & \textcolor{gray}{\textbf{0.683$\pm$0.031}} & \textcolor{gray}{\textbf{0.729$\pm$0.035}} & 0.787$\pm$0.031 & 0.863$\pm$0.034 & 2.20 \\
SKMM & \textbf{0.510$\pm$0.053} & \textbf{0.556$\pm$0.056} & \textbf{0.596$\pm$0.050} & \textbf{0.641$\pm$0.052} & \textbf{0.680$\pm$0.051} & \textbf{0.735$\pm$0.043} & \textbf{0.787$\pm$0.040} & \textbf{0.829$\pm$0.034} & \textbf{0.880$\pm$0.031} & \textbf{0.945$\pm$0.028} & 1.00 \\
\bottomrule
\end{tabular}}
\end{table}

\begin{table}[t]
\centering
\caption{HIV --- Mondrian calibration (mean $\pm$ std). Best is \textbf{bold}; second best is \textcolor{gray}{\textbf{gray bold}}.}
\resizebox{\textwidth}{!}{\begin{tabular}{lccccccccccc}
\toprule
Method & 0.5 & 0.55 & 0.6 & 0.65 & 0.7 & 0.75 & 0.8 & 0.85 & 0.9 & 0.95 & Avg. Rank \\
\midrule
Uniform & 0.385$\pm$0.063 & 0.415$\pm$0.064 & 0.448$\pm$0.063 & 0.480$\pm$0.066 & 0.512$\pm$0.067 & 0.546$\pm$0.068 & 0.581$\pm$0.068 & 0.619$\pm$0.067 & 0.667$\pm$0.064 & 0.732$\pm$0.060 & 4.40 \\
KDE & 0.341$\pm$0.124 & 0.358$\pm$0.123 & 0.373$\pm$0.130 & 0.392$\pm$0.132 & 0.411$\pm$0.139 & 0.427$\pm$0.150 & 0.457$\pm$0.159 & 0.479$\pm$0.171 & 0.515$\pm$0.177 & 0.552$\pm$0.201 & 5.60 \\
CoDrugKDE & 0.409$\pm$0.096 & 0.445$\pm$0.103 & 0.481$\pm$0.106 & 0.509$\pm$0.109 & 0.544$\pm$0.106 & 0.574$\pm$0.107 & 0.615$\pm$0.119 & 0.659$\pm$0.115 & 0.705$\pm$0.112 & 0.769$\pm$0.097 & 3.40 \\
Logistic & 0.969$\pm$0.001 & 0.969$\pm$0.001 & 0.969$\pm$0.001 & 0.969$\pm$0.001 & 0.969$\pm$0.001 & 0.969$\pm$0.001 & 0.969$\pm$0.001 & \textcolor{gray}{\textbf{0.969$\pm$0.001}} & \textcolor{gray}{\textbf{0.969$\pm$0.001}} & \textbf{0.972$\pm$0.009} & 4.20 \\
KMM & \textcolor{gray}{\textbf{0.452$\pm$0.059}} & \textcolor{gray}{\textbf{0.488$\pm$0.060}} & \textcolor{gray}{\textbf{0.527$\pm$0.057}} & \textcolor{gray}{\textbf{0.565$\pm$0.048}} & \textcolor{gray}{\textbf{0.600$\pm$0.044}} & \textcolor{gray}{\textbf{0.640$\pm$0.039}} & \textcolor{gray}{\textbf{0.674$\pm$0.035}} & 0.714$\pm$0.038 & 0.759$\pm$0.032 & 0.826$\pm$0.030 & 2.30 \\
SKMM & \textbf{0.505$\pm$0.059} & \textbf{0.548$\pm$0.054} & \textbf{0.593$\pm$0.053} & \textbf{0.631$\pm$0.048} & \textbf{0.676$\pm$0.053} & \textbf{0.713$\pm$0.048} & \textbf{0.770$\pm$0.040} & \textbf{0.816$\pm$0.043} & \textbf{0.861$\pm$0.034} & \textcolor{gray}{\textbf{0.910$\pm$0.031}} & 1.10 \\
\bottomrule
\end{tabular}}
\end{table}

\begin{table}[t]
\centering
\caption{Tox21 --- Global calibration (mean $\pm$ std). Best is \textbf{bold}; second best is \textcolor{gray}{\textbf{gray bold}}.}
\resizebox{\textwidth}{!}{\begin{tabular}{lccccccccccc}
\toprule
Method & 0.5 & 0.55 & 0.6 & 0.65 & 0.7 & 0.75 & 0.8 & 0.85 & 0.9 & 0.95 & Avg. Rank \\
\midrule
Uniform & 0.142$\pm$0.014 & 0.159$\pm$0.012 & 0.187$\pm$0.018 & 0.209$\pm$0.018 & 0.236$\pm$0.019 & 0.291$\pm$0.020 & 0.352$\pm$0.009 & 0.460$\pm$0.015 & 0.586$\pm$0.010 & 0.760$\pm$0.029 & 5.80 \\
KDE & 0.249$\pm$0.021 & 0.270$\pm$0.023 & 0.308$\pm$0.023 & 0.331$\pm$0.023 & 0.358$\pm$0.017 & 0.397$\pm$0.023 & 0.434$\pm$0.032 & 0.480$\pm$0.037 & 0.564$\pm$0.028 & 0.709$\pm$0.070 & 5.20 \\
CoDrugKDE & 0.408$\pm$0.117 & 0.444$\pm$0.116 & 0.480$\pm$0.125 & 0.537$\pm$0.129 & 0.582$\pm$0.145 & 0.628$\pm$0.132 & 0.682$\pm$0.137 & 0.769$\pm$0.066 & 0.846$\pm$0.038 & 0.906$\pm$0.027 & 3.90 \\
Logistic & 0.570$\pm$0.063 & 0.613$\pm$0.076 & 0.671$\pm$0.094 & 0.731$\pm$0.117 & 0.780$\pm$0.101 & 0.839$\pm$0.071 & 0.902$\pm$0.036 & 0.918$\pm$0.034 & \textcolor{gray}{\textbf{0.931$\pm$0.021}} & \textcolor{gray}{\textbf{0.940$\pm$0.009}} & 2.80 \\
KMM & \textcolor{gray}{\textbf{0.454$\pm$0.027}} & \textcolor{gray}{\textbf{0.509$\pm$0.021}} & \textcolor{gray}{\textbf{0.551$\pm$0.005}} & \textcolor{gray}{\textbf{0.591$\pm$0.008}} & \textcolor{gray}{\textbf{0.626$\pm$0.015}} & \textcolor{gray}{\textbf{0.687$\pm$0.006}} & \textcolor{gray}{\textbf{0.748$\pm$0.011}} & \textcolor{gray}{\textbf{0.807$\pm$0.013}} & 0.847$\pm$0.015 & 0.905$\pm$0.013 & 2.30 \\
SKMM & \textbf{0.471$\pm$0.024} & \textbf{0.519$\pm$0.014} & \textbf{0.574$\pm$0.033} & \textbf{0.643$\pm$0.069} & \textbf{0.695$\pm$0.056} & \textbf{0.761$\pm$0.028} & \textbf{0.798$\pm$0.018} & \textbf{0.834$\pm$0.024} & \textbf{0.892$\pm$0.014} & \textbf{0.956$\pm$0.004} & 1.00 \\
\bottomrule
\end{tabular}}
\end{table}

\begin{table}[t]
\centering
\caption{Tox21 --- Mondrian calibration (mean $\pm$ std). Best is \textbf{bold}; second best is \textcolor{gray}{\textbf{gray bold}}.}
\resizebox{\textwidth}{!}{\begin{tabular}{lccccccccccc}
\toprule
Method & 0.5 & 0.55 & 0.6 & 0.65 & 0.7 & 0.75 & 0.8 & 0.85 & 0.9 & 0.95 & Avg. Rank \\
\midrule
Uniform & 0.137$\pm$0.017 & 0.154$\pm$0.013 & 0.175$\pm$0.010 & 0.202$\pm$0.017 & 0.225$\pm$0.019 & 0.257$\pm$0.019 & 0.317$\pm$0.013 & 0.386$\pm$0.009 & 0.518$\pm$0.014 & 0.623$\pm$0.011 & 6.00 \\
KDE & 0.247$\pm$0.019 & 0.268$\pm$0.021 & 0.305$\pm$0.023 & 0.326$\pm$0.029 & 0.353$\pm$0.016 & 0.386$\pm$0.018 & 0.429$\pm$0.032 & 0.469$\pm$0.030 & 0.542$\pm$0.039 & 0.669$\pm$0.078 & 5.00 \\
CoDrugKDE & 0.399$\pm$0.121 & 0.438$\pm$0.115 & 0.473$\pm$0.126 & 0.523$\pm$0.130 & 0.575$\pm$0.149 & 0.627$\pm$0.145 & 0.678$\pm$0.139 & 0.751$\pm$0.104 & 0.845$\pm$0.038 & 0.906$\pm$0.027 & 4.00 \\
Logistic & 0.590$\pm$0.062 & 0.636$\pm$0.086 & 0.692$\pm$0.100 & 0.760$\pm$0.108 & 0.807$\pm$0.087 & 0.873$\pm$0.041 & 0.903$\pm$0.034 & 0.916$\pm$0.033 & \textcolor{gray}{\textbf{0.925$\pm$0.020}} & \textcolor{gray}{\textbf{0.934$\pm$0.008}} & 2.80 \\
KMM & \textcolor{gray}{\textbf{0.448$\pm$0.022}} & \textcolor{gray}{\textbf{0.504$\pm$0.033}} & \textcolor{gray}{\textbf{0.546$\pm$0.005}} & \textcolor{gray}{\textbf{0.585$\pm$0.017}} & \textcolor{gray}{\textbf{0.616$\pm$0.010}} & \textcolor{gray}{\textbf{0.675$\pm$0.025}} & \textcolor{gray}{\textbf{0.743$\pm$0.012}} & \textcolor{gray}{\textbf{0.809$\pm$0.012}} & 0.860$\pm$0.016 & 0.907$\pm$0.014 & 2.20 \\
SKMM & \textbf{0.466$\pm$0.025} & \textbf{0.521$\pm$0.013} & \textbf{0.570$\pm$0.029} & \textbf{0.627$\pm$0.051} & \textbf{0.688$\pm$0.048} & \textbf{0.767$\pm$0.017} & \textbf{0.796$\pm$0.018} & \textbf{0.830$\pm$0.016} & \textbf{0.891$\pm$0.015} & \textbf{0.951$\pm$0.003} & 1.00 \\
\bottomrule
\end{tabular}}
\end{table}

\begin{table}[t]
\centering
\caption{hERG\_Karim --- Global calibration (mean $\pm$ std). Best is \textbf{bold}; second best is \textcolor{gray}{\textbf{gray bold}}.}
\resizebox{\textwidth}{!}{\begin{tabular}{lccccccccccc}
\toprule
Method & 0.5 & 0.55 & 0.6 & 0.65 & 0.7 & 0.75 & 0.8 & 0.85 & 0.9 & 0.95 & Avg. Rank \\
\midrule
Uniform & 0.564$\pm$0.026 & 0.600$\pm$0.024 & 0.643$\pm$0.034 & 0.692$\pm$0.032 & 0.734$\pm$0.037 & 0.776$\pm$0.045 & 0.820$\pm$0.044 & 0.862$\pm$0.037 & 0.912$\pm$0.024 & 0.958$\pm$0.014 & 4.30 \\
KDE & 0.700$\pm$0.180 & 0.700$\pm$0.180 & 0.700$\pm$0.180 & 0.700$\pm$0.180 & \textbf{0.700$\pm$0.180} & 0.709$\pm$0.170 & 0.733$\pm$0.151 & 0.737$\pm$0.149 & 0.755$\pm$0.157 & 0.807$\pm$0.177 & 5.50 \\
CoDrugKDE & 0.551$\pm$0.037 & 0.592$\pm$0.036 & 0.636$\pm$0.045 & 0.690$\pm$0.042 & 0.737$\pm$0.057 & 0.779$\pm$0.060 & 0.822$\pm$0.053 & 0.863$\pm$0.046 & \textcolor{gray}{\textbf{0.907$\pm$0.032}} & \textbf{0.949$\pm$0.021} & 3.90 \\
Logistic & \textcolor{gray}{\textbf{0.535$\pm$0.043}} & \textcolor{gray}{\textbf{0.578$\pm$0.039}} & \textcolor{gray}{\textbf{0.622$\pm$0.043}} & \textcolor{gray}{\textbf{0.674$\pm$0.037}} & 0.719$\pm$0.028 & \textcolor{gray}{\textbf{0.765$\pm$0.030}} & 0.813$\pm$0.035 & 0.861$\pm$0.027 & 0.910$\pm$0.027 & 0.957$\pm$0.015 & 2.50 \\
KMM & 0.545$\pm$0.027 & 0.587$\pm$0.038 & 0.645$\pm$0.045 & 0.686$\pm$0.057 & 0.728$\pm$0.055 & 0.776$\pm$0.055 & \textcolor{gray}{\textbf{0.812$\pm$0.060}} & \textcolor{gray}{\textbf{0.844$\pm$0.074}} & 0.883$\pm$0.066 & 0.930$\pm$0.065 & 3.60 \\
SKMM & \textbf{0.523$\pm$0.046} & \textbf{0.570$\pm$0.053} & \textbf{0.606$\pm$0.067} & \textbf{0.645$\pm$0.067} & \textcolor{gray}{\textbf{0.687$\pm$0.058}} & \textbf{0.747$\pm$0.053} & \textbf{0.795$\pm$0.046} & \textbf{0.846$\pm$0.028} & \textbf{0.899$\pm$0.018} & \textcolor{gray}{\textbf{0.944$\pm$0.016}} & 1.20 \\
\bottomrule
\end{tabular}}
\end{table}

\begin{table}[t]
\centering
\caption{hERG\_Karim --- Mondrian calibration (mean $\pm$ std). Best is \textbf{bold}; second best is \textcolor{gray}{\textbf{gray bold}}.}
\resizebox{\textwidth}{!}{\begin{tabular}{lccccccccccc}
\toprule
Method & 0.5 & 0.55 & 0.6 & 0.65 & 0.7 & 0.75 & 0.8 & 0.85 & 0.9 & 0.95 & Avg. Rank \\
\midrule
Uniform & 0.594$\pm$0.044 & 0.629$\pm$0.040 & 0.664$\pm$0.037 & 0.696$\pm$0.036 & 0.729$\pm$0.037 & 0.768$\pm$0.038 & 0.812$\pm$0.032 & \textbf{0.853$\pm$0.019} & \textbf{0.899$\pm$0.015} & \textcolor{gray}{\textbf{0.941$\pm$0.009}} & 4.30 \\
KDE & \textbf{0.545$\pm$0.129} & \textbf{0.545$\pm$0.129} & \textbf{0.591$\pm$0.115} & 0.612$\pm$0.115 & 0.620$\pm$0.120 & 0.642$\pm$0.125 & 0.642$\pm$0.125 & 0.662$\pm$0.140 & 0.684$\pm$0.137 & 0.716$\pm$0.071 & 4.30 \\
CoDrugKDE & 0.571$\pm$0.043 & 0.612$\pm$0.045 & 0.649$\pm$0.042 & 0.681$\pm$0.044 & 0.720$\pm$0.045 & 0.759$\pm$0.041 & \textbf{0.801$\pm$0.035} & \textcolor{gray}{\textbf{0.840$\pm$0.026}} & 0.887$\pm$0.014 & 0.929$\pm$0.010 & 3.20 \\
Logistic & \textcolor{gray}{\textbf{0.546$\pm$0.046}} & 0.590$\pm$0.043 & 0.631$\pm$0.045 & \textcolor{gray}{\textbf{0.675$\pm$0.036}} & \textcolor{gray}{\textbf{0.716$\pm$0.033}} & \textcolor{gray}{\textbf{0.758$\pm$0.030}} & 0.811$\pm$0.024 & 0.861$\pm$0.018 & \textcolor{gray}{\textbf{0.909$\pm$0.018}} & \textbf{0.950$\pm$0.011} & 2.50 \\
KMM & 0.572$\pm$0.058 & 0.609$\pm$0.057 & 0.656$\pm$0.062 & 0.694$\pm$0.072 & 0.726$\pm$0.075 & 0.764$\pm$0.085 & \textcolor{gray}{\textbf{0.802$\pm$0.090}} & 0.831$\pm$0.093 & 0.868$\pm$0.089 & 0.914$\pm$0.069 & 4.40 \\
SKMM & 0.548$\pm$0.078 & \textcolor{gray}{\textbf{0.589$\pm$0.074}} & \textcolor{gray}{\textbf{0.625$\pm$0.069}} & \textbf{0.666$\pm$0.067} & \textbf{0.706$\pm$0.063} & \textbf{0.753$\pm$0.062} & 0.797$\pm$0.055 & 0.839$\pm$0.040 & 0.884$\pm$0.029 & 0.931$\pm$0.023 & 2.30 \\
\bottomrule
\end{tabular}}
\end{table}

\end{document}